\documentclass{article}

\usepackage[final,nonatbib]{neurips_2022}
\usepackage{graphicx}
\usepackage{amsmath,amsfonts,amssymb,amsthm,dsfont,mathrsfs}
\usepackage{bbm}
\usepackage{mathbbol}
\usepackage{mathtools}%
\usepackage[numbers]{natbib}
\usepackage{url}
\usepackage{color}
\usepackage{epstopdf}
\usepackage{dsfont}
\usepackage{ifpdf}
\usepackage{float}
\usepackage{bbm}
\usepackage{caption}
\usepackage{subcaption}

\newcommand*{\Scale}[2][4]{\scalebox{#1}{$#2$}}%

\usepackage[utf8]{inputenc} 
\usepackage[T1]{fontenc}    
\usepackage{hyperref}       
\usepackage{url}            
\usepackage{booktabs}       
\usepackage{amsfonts}       
\usepackage{nicefrac}       
\usepackage{microtype}      
\usepackage{xcolor}         
\newcommand{\dff}{\stackrel{\scriptscriptstyle\triangle}{=}}
\newtheorem{theorem}{Theorem}
\newtheorem{lemma}{Lemma}

\newtheorem{remark}{Remark}
\newtheorem{definition}{Definition}

\def \bX {{\mathbf{X}}}
\def \bx {{\mathbf{x}}}
\def \bz {{\mathbf{z}}}

\def \bC {{\mathbf{C}}}
\def \bV {{\mathbf{V}}}
\def \bU {{\mathbf{U}}}
\def \bD {{\mathbf{D}}}
\def \bW {{\mathbf{W}}}
\def \bJ {{\mathbf{J}}}

\def \bL {{\mathbf{L}}}

\def \bH {{\mathbf{H}}}
\def \bP {{\mathbf{P}}}
\def \by {{\mathbf{y}}}
\def \br {{\mathbf{r}}}
\def \bv{{\mathbf{v}}}
\def \bu {{\mathbf{u}}}
\def \bP {{\mathbf{P}}}
\def \bI {{\mathbf{I}}}
\def \mR {{\mathbb{R}}}
\def \mE {{\mathbb{E}}}
\def \mP {{\mathbb{P}}}
\def \bE {{\mathbf{E}}}
\def \cH {{\mathcal{H}}}
\title{coVariance Neural Networks}

\author{%
  Saurabh Sihag\\
 University of Pennsylvania\\
  \texttt{sihags@pennmedicine.upenn.edu} \\
   \And
   Gonzalo Mateos \\
   University of Rochester \\
   \texttt{gmateosb@ece.rochester.edu} \\
   \AND
   Corey McMillan \\
   University of Pennsylvania \\
   \texttt{cmcmilla@pennmedicine.upenn.edu} \\
   \And
   Alejandro Ribeiro \\
   University of Pennsylvania \\
   \texttt{aribeiro@seas.upenn.edu} \\
}

\begin{document}

\maketitle

\begin{abstract}
  Graph neural networks (GNN) are an effective framework that exploit inter-relationships within graph-structured data for learning. Principal component analysis (PCA) involves the projection of data on the eigenspace of the covariance matrix and draws similarities with the graph convolutional filters in GNNs. Motivated by this observation, we study a GNN architecture, called coVariance neural network (VNN), that operates on sample covariance matrices as graphs. We theoretically establish the stability of VNNs to perturbations in the covariance matrix, thus, implying an advantage over standard PCA-based data analysis approaches that are prone to instability due to principal components associated with close eigenvalues. Our experiments on real-world datasets validate our theoretical results and show that VNN performance is indeed more stable than PCA-based statistical approaches. Moreover, our experiments on multi-resolution datasets also demonstrate that VNNs are amenable to transferability of performance over covariance matrices of different dimensions; a feature that is infeasible for PCA-based approaches.
\end{abstract}

\section{Introduction}
Convolutional neural networks (CNN) are among the most popular deep learning (DL) architectures due to their demonstrated ability to learn latent features from Euclidean data adaptively and automatically~\cite{lecun2015deep}. Motivated by the fact that graph models can naturally represent data in many practical applications~\cite{scarselli2008graph}, the generalizations of CNNs to adapt to graph data using GNN architectures have rapidly emerged in recent years~\cite{wu2020comprehensive}. A practical challenge to DL architectures is scalability to high-dimensional data~\cite{garg2019low}. Specifically, DL models can typically have millions of parameters that require long training times and high storage capacity~\cite{hosseini2020deep}. GNN architectures have two attractive features towards learning from high-dimensional data: 1. the number of parameters associated with graph convolution operations in GNNs is independent of the graph size; and 2. under certain assumptions, GNN performance can be transferred across graphs of different dimensions~\cite{ruiz2020graphon}.

Principal component analysis (PCA) is a traditional, non-parametric approach for linear dimensionality reduction and increase of interpretability of high dimensional datasets~\cite{jolliffe2016principal}. Accordingly, it is common to embed the PCA transform within a DL pipeline~\cite{garg2019low, chan2015pcanet, umer2020fake}. Such approaches usually use PCA to optimize the DL model parameters by removing redundancy in the features in intermediate layers. The scope of our work is different. In this paper, we leverage the connections between PCA and the mathematical concepts used in GNNs to propose a learning framework that can recover PCA and also, inherit additional desirable features associated with GNNs.

\subsection{PCA and Graph Fourier Transform}\label{pca_gft_intro}
PCA is an orthonormal transform which identifies the modes of maximum variance in the dataset. The transformed dimensions of the data (known as principal components) are uncorrelated. Computing the PCA transform reduces to finding the eigenbasis of the covariance matrix or equivalently, the singular value decomposition of the data matrix~\cite{shlens2014tutorial}. We note that PCA computation using eigendecomposition of a covariance matrix has similarities with the graph Fourier transform (GFT) in the graph signal processing literature~\cite{ortega2018graph}. This observation is discussed more formally next.


 Consider an $m-$dimensional random vector $\bX \in \mR^{m \times 1}$ with the true or ensemble covariance matrix of $\bX$ defined as $\bC \dff \mE[(\bX-\mE[\bX])(\bX-\mE[\bX])^{\sf T}]$.  In practice, we may not observe the data model for $\bX$ and its statistics, such as the covariance matrix $\bC$ directly. Alternatively, we have access to a dataset consisting of $n$ random, independent and identically distributed (i.i.d) samples of $\bX$, given by~$\bx_i \in \mR^{m \times 1},\forall i\in\{1,\dots,n\}$, where the data matrix is $\hat\bx_n = [\bx_1,\dots,\bx_n]$. Using the samples $\bx_i, \forall i\in\{1,\dots,n\}$, we can form an estimate of the ensemble covariance matrix, conventionally referred to as the sample covariance matrix
\begin{align}\label{sample_cov}
    \hat \bC_{n} \triangleq \frac{1}{n} \sum\limits_{i=1}^n(\bx_i - \bar\bx) (\bx_i-\bar\bx)^{\sf T}\;,
\end{align}
 where $\bar \bx$ is the sample mean across $n$ samples and $\cdot^{\sf T}$ refers to the transpose operator. Given the eigendecomposition of $\hat\bC_n$
\begin{align}\label{sample_eig}
    \hat \bC_{n} = \bU \bW \bU^{\sf T}\;,
\end{align}
where $\bU = [\bu_1,\dots,\bu_m]$ is the matrix with its columns as the eigenvectors and $\bW= {\sf diag}(w_1,\dots,w_m)$ is the diagonal matrix of eigenvalues, such that, $w_1\geq w_2\cdots\geq w_m$, the implementation of PCA is given by~\cite{shlens2014tutorial}
\begin{align}\label{pca_tr}
    \hat\by_n = \bU^{\sf T}\hat\bx_n\;. 
\end{align}
The eigenvectors of $\hat\bC_n$ form the principal components and the transformation in~\eqref{pca_tr} is equivalent to the projection of the sample $\bx_i$ in the eigenvector space of $\hat\bC_n$. A detailed overview of PCA is provided in Appendix~\ref{pca_ov}.  We note that in the context of graph signal processing,~\eqref{pca_tr} is equivalent to the graph Fourier transform (GFT) if the covariance matrix $\hat\bC_n$ is considered the graph representation or a graph shift operator for the data $\hat\bx_n$~\cite{gama2020graphs}. 
\subsection{Contributions}
Motivated by the observation in Section~\ref{pca_gft_intro}, we introduce the notion of coVariance Fourier transform and define coVariance filters similar to graph convolutional filters in GNNs. Our analysis shows that standard PCA can be recovered using coVariance filters. Furthermore, we propose a DL architecture based on coVariance filters, referred to as coVariance neural networks (VNN)\footnote{We emphasize on V in coVariance for abbreviations to avoid confusions with any existing machine learning frameworks.}: a permutation-invariant architecture derived from the GNN framework that uses convolutional filters defined on the covariance matrix. Understanding VNNs is relevant given the ubiquity of so-termed correlation networks~\cite[Chapter 7.3.1]{kolaczyk2009book} and covariance matrices as graph models for neuroimaging~\cite{bessadok2021graph}, user preference~\cite{huang2018}, financial~\cite{cardoso2020algorithms}, and gene expression data\cite[Chapter 7.3.4]{kolaczyk2009book}, to name a few timely domains. Summary of our contributions is as follows:

{\bf Stability of VNN:} For finite $n$, the eigenvectors of the sample covariance matrix $\hat\bC_n$ are perturbed from the eigenvectors of the ensemble covariance matrix $\bC$. Moreover, the principal components corresponding to eigenvalues that are close are unstable, i.e., small changes in the dataset may cause large changes in such principal components~\cite{joliffe1992principal} and subsequently, irreproducible statistical results~\cite{elhaik2022principal}.  Statistical learning approaches such as regression, classification, clustering that use PCA as the first step for dimension reduction inherit this instability. Our main theoretical contribution is to establish the~\emph{stability} of VNNs to perturbations in the sample covariance matrix in terms of sample size $n$. For this purpose, we leverage the perturbation theory for sample covariance matrices and show that for a given $\hat\bC_n$, the perturbation in the VNN output $\Phi(\hat\bC_n)$ with respect to the VNN output for $\bC$, $\Phi(\bC)$ satisfies $\|\Phi(\hat\bC_n) - \Phi(\bC)\| = {\cal O}(1/n^{\frac{1}{2}-\varepsilon})$ for some $\varepsilon < 1/2$ with high probability. Therefore, our analysis ties the stability of VNNs to the quality of sample covariance matrices, which is distinct from existing studies that study stability of GNNs under abstract, data-independent discrete or geometric perturbations to graphs~\cite{gama2020stability, keriven2020convergence}.

{\bf Empirical Validations:} Our experiments on real-world and synthetic datasets show that VNNs are indeed stable to perturbations in the sample covariance matrix. Moreover, on a multi-resolution dataset, we also illustrate the \emph{transferability} of VNN performance to datasets of different dimensions without re-training. These observations provide convincing evidence for the advantages offered by VNNs over standard PCA-based approaches for data analysis. 
\subsection{Related Work}
\textbf{GNNs}: GNNs broadly encapsulate all DL models adapted to graph data~\cite{wu2020comprehensive, abadal2021computing}. Recent survey articles in~\cite{wu2020comprehensive} and~\cite{abadal2021computing} categorize the variants of GNNs according to diverse criteria that include mathematical formulations, algorithms, and hardware/software implementations. Convolutional GNNs~\cite{zhang2019graph}, graph autoencoders~\cite{kipf2016variational}, recurrent GNNs, and gated GNNs~\cite{li2016gated} are among a few prominently studied and applied categories of GNNs. The taxonomy pertinent to this paper is that of graph convolutional network (GCN) that generalizes the concept of convolution to graphs. Specifically, VNNs are based on GCNs that model the convolution operation via graph convolutional filters~\cite{ruiz2021graph}. The graph convolutional filters are defined as polynomials over the graph representation where the order of the polynomial controls the aggregation of information over nodes in the graph~\cite{gama2020graphs}. We extend the definition of graph convolutional filters to filters over sample covariance matrices to define VNNs. The ubiquity of PCA and prevalence of covariance matrices across disciplines for data analysis motivate us to study VNNs independently of aforementioned GCNs. Moreover, the unique structure inherent to covariance matrices (relative to arbitrary graphs) and their sample-based perturbations facilitates sharper stability analyses, which offer novel insights.

\textcolor{black}{\textbf{Robust PCA:} Robust PCA is a variant of PCA that aims to recover the low rank representation of a dataset arbitrarily corrupted due to outliers~\cite{candes2011robust,5513535,xu2010robust}. The outliers considered in robust PCA approaches could manifest in the dataset due to missing data, adversarial behavior, defects or contamination in data collection. In contrast to such settings, here we consider the inherent statistical uncertainty in the sample covariance matrix, where ill-defined eigenvalues and eigenvectors of the sample covariance matrix can be the source of instability in statistical inference. Therefore, the notion of stability of VNNs in this paper is fundamentally distinct from the notion of stability or robustness in robust PCA approaches. }

\section{coVariance Filters}\label{fltrsec}
If $m$ dimensions of data $\hat\bx_n$ can be represented as the nodes of an $m$-node, undirected graph, the sample covariance matrix $\hat\bC_n$ is equivalent to its adjacency matrix. In GNNs, graph Fourier transform projects graph signals in the eigenspace of the graph and is leveraged to analyze graph convolutional filters~\cite{gama2020graphs}. Therefore, we start by formalizing the notion of coVariance Fourier transform (abbreviated as VFT). For this purpose, we leverage the eigendecomposition of $\hat\bC_n$ in~\eqref{sample_eig}.
\begin{definition} [coVariance Fourier Transform]
Consider a sample covariance matrix $\hat\bC_n$ as defined in~\eqref{sample_cov}. The coVariance Fourier transform (VFT) of a random sample $\bx$ is defined as its projection on the eigenspace of $\hat\bC_n$ and is given by
\begin{align}\label{vft}
    \tilde\bx \dff \bU^{\sf T} \bx\;.
\end{align}
\end{definition}
The $i$-th entry of $\tilde\bx$, i.e., $[\tilde\bx]_i$ represents the $i$-th Fourier coefficient and is associated with the eigenvalue~$w_i$. Note that the similarity between PCA and VFT implies that eigenvalue~$w_i$ encodes the variability of dataset $\bx_n$ in the direction of the principal component $\bu_i$. In this context, the eigenvalues of the covariance matrix are mathematical equivalent of the notion of graph frequencies in graph signal processing~\cite{ortega2018graph}. Next, we define the notion of coVariance filters (VF) that are polynomials in the covariance matrix.
\begin{definition}[coVariance Filter]
Given a set of real valued parameters $\{h_k\}_{k= 0}^{m}$, the coVariance filter for the covariance matrix $\hat\bC_n$ is defined as 
\begin{align} 
    \bH(\hat\bC_n) \dff \sum\limits_{k=0}^{m} h_k\hat\bC_n^k\;.
\end{align}
The output of the covariance filter  $\bH(\hat\bC_n)$ for an input $\bx$ is given by
\begin{align}\label{vf}
    \bz = \sum\limits_{k=0}^{m} h_k \hat\bC^k_n \bx = \bH(\hat\bC_n) \bx\;.
\end{align}
\end{definition}
 The coVariance filter $\bH(\hat\bC_n)$ follows similar analytic concepts of combining information in different neighborhoods as in the well-studied graph convolutional filters~\cite{gama2020graphs}. Moreover, the filter is defined by the parameters $\{h_k\}_{k=0}^m$ and therefore, for the ensemble covariance matrix $\bC$, the coVariance filter is given by $\bH(\bC)$. On taking the VFT of the output in~\eqref{vf} and leveraging~\eqref{sample_eig}, we have
\begin{align}
   \tilde \bz\dff  \bU^{\sf T} \bz &=  \bU^{\sf T} \sum\limits_{k=0}^{m} h_k [\bU \bW\bU^{\sf T}]^k \bx\enskip= \sum\limits_{k=0}^{m} h_k \bW^k \tilde\bx \label{vf3}\;,
\end{align}
where~\eqref{vf3} holds from the orthonormality of eigenvectors and definition of VFT in~\eqref{vft}. 
Using~\eqref{vf3}, we can further define the frequency response of the coVariance filter over the covariance matrix $\hat\bC_n$ in the domain of its principal components as
\begin{align}\label{vvf}
    h(w_i) = \sum\limits_{k=0}^{m} h_k w_i^k\;,
\end{align}
such that, from~\eqref{vf3} and~\eqref{vvf}, the $i$-th element of $\tilde \bz_i$ has the following relationship
\begin{align}\label{vfpca}
    [\tilde \bz]_i = h(w_i) [\tilde\bx]_i\;.
\end{align}
Equation~\eqref{vfpca} reveals that performing the coVariance filtering operation boils down to processing (e.g., amplifying or attenuating) the principal components of the data. This observation draws analogy with the linear-time invariant systems in signal processing where the different frequency modes (in this case, principal components) can be processed separately using coVariance filters, in a way determined by the frequency response values $h(w_i)$. 
For instance, using a narrowband coVariance filter whose frequency response is $h(\lambda) = 1$, if $\lambda = w_i$ and $h(\lambda) = 0$, otherwise, we recover the score corresponding to the projection of $\bx$ on $\bu_i$, i.e, the $i$-th principal component of $\hat\bC_n$. Therefore, there exist filterbanks of narrowband coVariance filters that enable the recovery of the PCA transformation. This observation is formalized in Theorem~\ref{pca_cf}. 
\begin{theorem}[coVariance Filter Implementation of PCA]\label{pca_cf}
Given a covariance matrix $\hat\bC_n$ with eigendecomposition in~\eqref{sample_eig}, if the PCA transformation of input $\bx$ is given by $\by = \bU^{\sf T} \bx$, there exists a filterbank of coVariance filters $\{\bH_i(\hat\bC_n): i\in \{1,\dots,m\}\}$, such that, the score of the projection of input $\bx$ on eigenvector~$\bu_i$ can be recovered by the application of a coVariance filter $\bH_i(\hat\bC_n)$ as:
\begin{align}
    [\by]_i = \bu_i^{\sf T}  \bH_i(\hat\bC_n) \bx \;, 
\end{align}
where the frequency response $h_i(\lambda)$ of the filter $\bH_i(\hat\bC_n)$ is given by
\begin{equation}
    h_i(\lambda) = \begin{cases}
     \eta_i ,\quad \text{if} \quad \lambda =w_i  \;,\\
     0, \quad \text{otherwise}
    \end{cases}\;.
\end{equation}
\end{theorem}
Theorem~\ref{pca_cf} establishes equivalence between processing data samples with PCA and processing data samples with a specific polynomial on the covariance matrix. As we shall see in subsequent sections, the processing on a polynomial of covariance matrix has advantages in terms of stability with respect to the perturbations in the sample covariance matrix. The design of frequency response of different coVariance filters sufficient to implement PCA transformation according to Theorem~\ref{pca_cf} is shown figuratively in Appendix~\ref{vnn_arch}. If it is desired to have PCA-based data transformation be followed by dimensionality reduction or statistical learning tasks such as regression or classification, the coVariance filter-based PCA can be coupled with the post-hoc analysis model to have an end-to-end framework that enables the optimization of the parameters $\eta_i$ in the frequency response of the coVariance filters.

\section{coVariance Neural Networks (VNN)}
In this section, we propose coVariance Neural Network (VNN), which provides an end-to-end, non-linear, parametric mapping from the input data $\bx$ to any generic objective $\br$  and is defined as
\begin{align}
    \br = \Phi(\bx; \hat\bC_n, \cH)\;,
\end{align}
for sample covariance matrix $\hat\bC_n$ where $\cH$ is the set of filter coefficients that characterize the representation space defined by the mapping~$\Phi(\cdot)$. The VNN $\Phi(\bx;\hat\bC_n, \cH)$ may be formed by multiple layers, where each layer consists of two main components: i) a filter bank made of VFs similar to that in~\eqref{vf}; and ii) a pointwise non-linear activation function (such as ${\sf ReLU}, \tanh$). Therefore, in principle, the architecture of VNNs is similar to that of graph neural networks with the covariance matrix $\hat\bC_n$ as the graph shift operator~\cite{gama2020graphs}. We next define the coVariance perceptron, which forms the building block of a VNN and is equivalent to a 1-layer VNN. 
\begin{definition}[coVariance Perceptron]
Consider a dataset with the sample covariance matrix $\hat \bC_n$. For a given non-linear activation function $\sigma(\cdot)$, input $\bx$, a coVariance filter~${\bH( \hat\bC_n) = \sum_{k=0}^m h_k \hat\bC_n^k}$ and its corresponding coefficient set $\cH $, the coVariance perceptron is defined as
\begin{align}\label{1l}
 \Phi(\bx; \hat\bC_n, \cH) \dff \sigma(\bH( \hat\bC_n)\bx)\;.
\end{align}
\end{definition}
The VNN can be constructed by cascading multiple layers of coVariance perceptrons (shown in Appendix~\ref{vnn_arch} in the Supplementary file). Note that the non-linear activation functions across different layers allow for non-linear transformations, thus, increasing the expressiveness of VNNs beyond linear mappings such as coVariance filters. Furthermore, similar to GNNs, we can further increase the representation power of VNNs by incorporating multiple parallel inputs and outputs per layer enabled by filter banks at every layer~\cite{gama2020graphs, goodfellow2016deep}. In this context, we remark that the definition of a one layer perceptron in~\eqref{1l} can be expanded to the following.
\begin{remark}[coVariance Perceptron with Filter Banks]\label{fltrbankvf}
Consider a coVariance perceptron with $F_{\sf in}$ $m$-dimensional inputs and $F_{\sf out}$ $m$-dimensional outputs. Denote the input at the perceptron by $\bx_{\sf in} = \Big[\bx_{\sf in}[1], \dots,\bx_{\sf in}[F_{\sf in} ]\Big]$ and the output at the perceptron by  $\bx_{\sf out} = \Big[\bx_{\sf out}[1], \dots,\bx_{\sf out}[F_{\sf out} ]\Big]$. Each input $\bx_{\sf in}[g], \forall g\in \{1,\dots, F_{\sf in} \}$ is processed by $F_{\sf out}$ coVariance filters in parallel. For $f\in \{1,\dots,F_{\sf out} \}$, the $f$-th output in $\bx_{\sf out}$ is given by
\begin{align}
   \bx_{\sf out}[f] &=  \sigma\left(\sum\limits_{g = 1}^{F_{\sf in} } \bH_{fg}(\hat\bC_n)\bx_{\sf in} [g] \right)\label{vfbnk}\enskip = \Phi(\bx_{\sf in};\hat\bC_n,\cH_f)\;,
\end{align}
where $\cH_f$ is the set of all filter coefficients in coVariance filters~$[\bH_{fg}(\hat\bC_n)]_{g}$ in~\eqref{vfbnk}. 
\end{remark}
Therefore, the VNN with filter bank implementation deploys $F_{\sf in}\times F_{\sf out} $ number of VFs in a layer defined by the covariance perceptron in Remark~\ref{fltrbankvf}. Next, we note that the definitions of coVariance filter in Section~\ref{fltrsec} and VNN are with respect to the sample covariance matrix $\hat\bC_n$. However, due to finite sample size effects, the sample covariance matrix will be a perturbed version of the ensemble covariance matrix $\bC$. PCA-based approaches that rely on eigendecomposition of the sample covariance matrix can potentially be highly sensitive to such perturbations. Specifically, if small changes are made to the dataset $\bx_n$, certain ill-defined eigenvalues and eigenvectors can induce instability in the outcomes of PCA-based statistical learning models~\cite{joliffe1992principal}. Therefore, it is desirable that the designs of coVariance filters and VNNs are not sensitive to random perturbations in the sample covariance matrix as compared to the ensemble covariance matrix $\bC$. In this context, we study the stability of coVariance filters and VNNs in the next section.

\section{Stability Analysis} \label{stability_all}
We start with the stability analysis of coVariance filters, which will also  be instrumental in establishing the stability of VNNs. Our results will leverage the eigendecomposition of $\bC$, given by
\begin{align}\label{eigenC}
    \bC = \bV \Lambda \bV^{\sf T}
\end{align}
where $\bV \in \mR^{m\times m}$ is the matrix of eigenvectors such that, $\bV = [\bv_1,\dots,\bv_m]$ and ${\Lambda = {\sf diag}(\lambda_1, \dots,\lambda_m)}$ is the diagonal matrix of eigenvalues, such that, $\lambda_1\geq \cdots\geq \lambda_m$.


\subsection{Stability of coVariance Filters}
To study the stability of coVariance filters, we analyze the effect of statistical uncertainty induced in the sample covariance matrix with respect to the ensemble covariance matrix on the coVariance filter output. To this end, we compare the outputs $\bH(\bC) \bx$ and $\bH(\hat\bC_n)\bx$ for any random instance $\bx$ of $\bX$. Without loss of generality, we present our results over instances of $\bX$ where $\|\bX\| \leq 1$. We also consider the following assumption on the frequency response.

 {\bf Assumption}: The frequency response of filter $\bH(\bC)$ satisfies:
 \begin{align}\label{fltr}
     |h(\lambda_i) - h(\lambda_j)| \leq M\frac{|\lambda_i - \lambda_j|}{k_i}\;,
 \end{align}
 where $k_i \dff \sqrt{\mE[\|\bX\bX^{\sf T} \bv_i\|^2] - \lambda_i^2}$, for some constant $M>0$ and for all non-zero eigenvalues $\lambda_i \neq \lambda_j, i,j \in \{1,\dots,m\}$ of $\bC$. Here, $k_i$ is a measure of kurtosis of the distribution of $\bX$.

 Next, we provide the result that establishes the stability of coVariance filters in Theorem~\ref{filterstab}. The impact of the assumption in~\eqref{fltr} on the filter stability is discussed subsequently.  To present the results in Theorem~\ref{filterstab}, we also use the following definitions:
 \begin{align}\label{addefs}
     k_{\sf min} \dff \min_{i\in \{1,\dots,m\}, \lambda_i>0} k_i \quad\text{and} \quad \kappa \dff \max_{i,j: \lambda_i \neq \lambda_j} \frac{k_i^2}{|\lambda_i - \lambda_j|}\;.
 \end{align}

\begin{theorem}[Stability of coVariance Filter]\label{filterstab}
Consider a random vector $\bX \in \mR^{m\times 1}$ , such that, its corresponding covariance matrix is given by $\bC = \mE[(\bX - \mE[\bX]) (\bX - \mE[\bX])^{\sf T}]$. For a sample covariance matrix $\hat\bC_n$ formed using $n$ i.i.d instances of $\bX$ and a random instance $\bx$ of $\bX$, such that, $\|\bx\| \leq 1$ and under assumption~\eqref{fltr}, the following holds with probability at least $1 - {n^{-2\varepsilon}} - 2\kappa m/n$ for any $\varepsilon \in (0,1/2]$:
\begin{align}\label{filterstab_rslt}
    \left\lVert \bH(\hat\bC_n) - \bH(\bC)\right\rVert = \frac{M}{n^{\frac{1}{2} - \varepsilon}}\cdot{\cal O}\left({\sqrt{m}} + \frac{\|\bC\|\sqrt{\log mn}}{k_{\sf min}n^{2\varepsilon}}\right)\;.
\end{align}
\end{theorem}
\begin{proof}
See Appendix~\ref{stab_pf}. 
\end{proof}
The right-hand side term in~\eqref{filterstab_rslt} and the conditions in the assumption in~\eqref{fltr} are obtained by the analysis of the finite sample size effect-driven perturbations in $\hat\bC_n$ and its eigenvectors with respect to that in $\bC$. From Theorem~\ref{filterstab}, we note that $ \|\bH(\hat\bC_n) - \bH(\bC)\|$ decays with the number of samples $n$ at least at the rate of $1/n^{\frac{1}{2} - \varepsilon}$. Thus, we conclude that the stability of the coVariance filter improves as the number of samples $n$ increases. This observation is along the expected lines as the estimate $\hat \bC_n$ becomes closer to the ensemble covariance matrix $\bC$ by the virtue of the law of large numbers. Next, we briefly discuss two aspects of the assumption in~\eqref{fltr}. Note that the upper bound in~\eqref{fltr} controls the variability of the frequency response $h(\lambda)$ with respect to $\lambda$. For any pair of eigenvalues $\lambda_i$ and $\lambda_j$ of $\bC$, this variability is tied to the eigengap $|\lambda_i - \lambda_j|$ and the factor $k_i$. We discuss this next.

\textbf{Discriminability between close eigenvalues}:
Firstly, for a given $k_i$ and eigenvalue $\lambda_i$, the response of the filter for any eigenvalue $\lambda_j, j\neq i$ becomes closer to $h(\lambda_i)$ if $|\lambda_i - \lambda_j|$ decreases. From the perturbation theory of eigenvectors and eigenvalues of sample covariance matrices, we know that the sample-based estimates of the eigenspaces corresponding to eigenvalues $\lambda_i$ and $\lambda_j$ become harder to distinguish as $|\lambda_i - \lambda_j|$ decreases~\cite{loukas2017close}. Therefore, the coVariance filter that satisfies~\eqref{fltr} sacrifices discriminability between close eigenvalues to preserve its stability with respect to the statistical uncertainty inherent in the sample covariance matrix.

\textbf{Stability with respect to kurtosis and estimation quality}: Since we have ${\mE[\|\bX\bX^{\sf T} \bv_i\|^2] \leq \mE[\|\bX\|^4]}$, the factor $k_i$ is tied to the measure of kurtosis of the underlying distribution of $\bX$ in the direction of $\bv_i$. Distributions with high kurtosis tend to have heavier tails and more outliers. Therefore, smaller $k_i$ indicates that the distribution of $\bX$ has a fast decaying tail in the direction of $\bv_i$ which allows for a more accurate estimation of $\lambda_i$ and $\bv_i$. We refer the reader to~\cite[Section 4.1.3]{loukas2017close} for additional details. In the context of coVariance filters, we note that the upper bound in~\eqref{fltr} is more liberal if $\lambda_i$ is associated with a smaller $k_i$ or equivalently, the distribution of $\bX$ has a smaller kurtosis in the direction of $\bv_i$. This observation implies that if $\lambda_i$ and $\bv_i$ are `easier' to estimate, the frequency response for $\lambda_j$ in the vicinity of $\lambda_i$ is less constrained. 

\subsection{Stability of coVariance Neural Networks}
The stability of VNNs is analyzed by comparing $\Phi(\bx; \hat\bC_n, \cH)$ and $\Phi(\bx;\bC,{\cal H})$. Note that stable graph convolutional filters imply the stability of GNNs for different perturbation models~\cite{gama2020stability}. In VNNs, the perturbations are derived from the finite sample effect in sample covariance matrix, which is distinct from the data-independent perturbation models considered in the existing literature on GNNs, and allow us to relate sample size $n$ with VNN stability. We formalize the stability of VNNs under the assumption of stable coVariance filters in Theorem~\ref{vnn_stab}. For this purpose, we consider a VNN $\Phi(\cdot;\cdot,{\cal H})$ with number of $m$-dimensional inputs and outputs per layer as $F$ and $L$ layers, with the filter bank given by ${\cal H} = \{\bH_{fg}^{\ell}\}, \forall f,g \in \{1,\dots, F\}, \ell \in \{1,\dots,L\}$. 

\begin{theorem}[Stability of VNN]\label{vnn_stab}
Consider a sample covariance matrix $\hat\bC_n$ and the ensemble covariance matrix $\bC$. Given a bank of coVariance filters $\{\bH_{fg}^{\ell}\}$, such that ${|h_{fg}^{\ell}(\lambda) |\leq 1}$ and a pointwise non-linearity $\sigma(\cdot)$, such that, $|\sigma(a) - \sigma(b)|\leq |a-b|$, if the covariance filters satisfy
\begin{align}
    \|\bH_{fg}^{\ell}(\hat\bC_n)- \bH_{fg}^{\ell}(\bC)\| \leq \alpha_n \;,
\end{align}
for some $\alpha_n > 0$, then, we have
\begin{align}
    \|\Phi(\bx;\hat\bC_n, \cH)-\Phi(\bx;\bC,\cH)\| \leq LF^{L-1} \alpha_n\;.
\end{align}
\end{theorem}
The proof of Theorem~\ref{vnn_stab} follows from~\cite[Appendix E]{gama2020graphs} for any generic $\alpha_n$. The parameter $\alpha_n$ represents the finite sample effect on the perturbations in $\hat\bC_n$ with respect to $\bC$. From Theorem~\ref{filterstab}, we note that $\alpha_n$ scales as $1/n^{\frac{1}{2} - \varepsilon}$ with respect to the number of samples $n$ for the coVariance filters whose frequency response depends on the eigengap and kurtosis of the underlying distribution of the data in~\eqref{fltr}. Furthermore, the stability of a VNN decreases with increase in number of $m$-dimensional inputs and outputs per layer $F$ and number of layers $L$, which is consistent with the stability properties of GNNs. Therefore, our results present a more holistic perspective to the stability of VNNs than that possible for generic GNNs. 
\textcolor{black}{
\begin{remark}[Computational Complexity of VNN]\label{comp}
For a coVariance perceptron defined in~(14), the computational cost is given by $O(m^2 T F_{\sf in}F_{\sf out})$, where $T\leq m$ is the maximum number of filter taps in its associated filter bank.
\end{remark}}
\textcolor{black}{From Remark~\ref{comp}, we note that the computational complexity of VNNs can be prohibitive for large~$m$. However, oftentimes sparsity is imposed as a regularization to estimate high-dimensional correlation matrices; see e.g.~\cite{bien2011sparse}. As a result, the computational complexity becomes $O(|E|TF_{\sf in}F_{\sf out})$, where $|E|$ is the number of non-zero correlations (edges in the covariance graph) and can be markedly smaller than $m^2$. Since VNN architecture is analogous to that of GNN, the property of GNN transferability (see ~\cite{ruiz2020graphon}) across different sized graphs can also establish scalability of VNNs to high-dimensional datasets for settings where multi-resolution datasets are available.} In the next section, we empirically validate our theoretical results on the stability of VNNs. Moreover, on a set of multi-resolution datasets, we also empirically evaluate VNNs for transferability. 

\section{Experiments}
In this section, we discuss our experiments on different datasets. Primarily, we evaluate VNNs on a regression problem on different neuroimaging datasets curated at University of Pennsylvania, where we regress human chronological age (time since birth) against cortical thickness data. Cortical thickness is a measure of the gray matter width and it is well established that cortical thickness varies with healthy aging~\cite{thambisetty2010longitudinal}. Additional details on the neurological utility of this experiment are included in Appendix~\ref{add_data}. Brief descriptions of these datasets are provided next. 



 {\bf ABC Dataset}: ABC dataset consists of the cortical thickness data collected from a heterogeneous population of $n=341$ subjects (mean age $=68.1$ years, standard deviation $=13$)  that consists of healthy adults, and subjects with mild cognitive impairment or Alzheimer's disease. For each individual, joint-label fusion~\cite{wang2012multi} was used to quantify mean cortical thickness in $m=104$ anatomical parcellations. Therefore, for every subject, we have a 104 dimensional vector whose entries correspond to the cortical thickness in distinct brain regions.

{\bf Multi-resolution FTDC Datasets}: FTDC Datasets consist of the cortical thickness data from $n=170$ healthy subjects (mean age $=64.26$ years, standard deviation $=8.26$). For each subject, the cortical thickness data is extracted according to a multiresolution Schaefer parcellation~\cite{schaefer2018local}, at 100 parcel, 300 parcel, and 500 parcel resolutions. Therefore, for each subject, we have the cortical thickness data consisting of $m=100$ features, $m=300$ features or $m=500$ features, with the higher number of features providing the cortical thickness data of a brain at a finer resolution. We leverage the different resolutions of data available to form three datasets:  FTDC100, FTDC300, and FTDC500, which form the cortical thickness datasets corresponding to 100, 300, and 500 features resolutions, respectively.

Our primary objective is to illustrate the stability and transferability of VNNs that also imply advantages over traditional PCA-based approaches. Hence, we use PCA-based regression as the primary baseline for comparison against VNNs. We use the ABC dataset to demonstrate the higher stability of VNNs over PCA-based regression models in Section~\ref{perturb}. In Section~\ref{transfer}, we use the FTDC datasets to demonstrate transferability of VNN performance across the multi-resolution datasets without re-training. The experiments in Section~\ref{transfer} clearly lay beyond the scope of PCA-based statistical learning models. Details on data and code availability for the experiments are included in Appendix~\ref{dataavail}.


\subsection{Stability against Perturbations in Sample Covariance Matrix}\label{perturb}
In this section, we evaluate the stability or robustness of the trained VNN and PCA-regression models against perturbations in the sample covariance matrix used in training. To this end, we first train nominal VNN and PCA-regression models. The effects of perturbations in the sample covariance matrix on the performance of nominal models are evaluated subsequently. We first describe the experiment designs for nominal models based on VNN and PCA-regression. 

{\bf VNN Experiments}: We randomly split ABC dataset into a $90/10$ train/test split. The sample covariance matrix is formed from $307$ samples in the training set, i.e., we have $\hat\bC_{307}$ of size $104\times 104$. The VNN consists of 2 layers with $2$ filter taps each, a filter bank of $13$ $m$-dimensional outputs per layer for $m=104$ dimensions of the input data, and a readout layer that calculates the unweighted mean of the outputs at the last layer to form an estimate for age. The hyperparameters for the VNN architecture and learning rate of the optimizer in this experiments and all subsequent VNN experiments in this section are chosen by a hyperparameter optimization framework called Optuna~\cite{akiba2019optuna}. The training set is randomly subdivided into subsets of $273$ and $34$ samples, such that, the VNN is trained with respect to the mean squared error loss between the predicted age and the true age in $273$ samples. The loss is optimized using batch stochastic gradient descent with Adam optimizer available in PyTorch library~\cite{paszke2019pytorch} for up to $100$ epochs. The learning rate for the optimizer is set to $0.0151$. The VNN model with the best minimum mean squared error performance on the remaining $34$ samples (which acts as a validation set) is included in the set of nominal models for this permutation of the training set.  
\begin{figure}[t]
  \centering
  \includegraphics[scale=0.4]{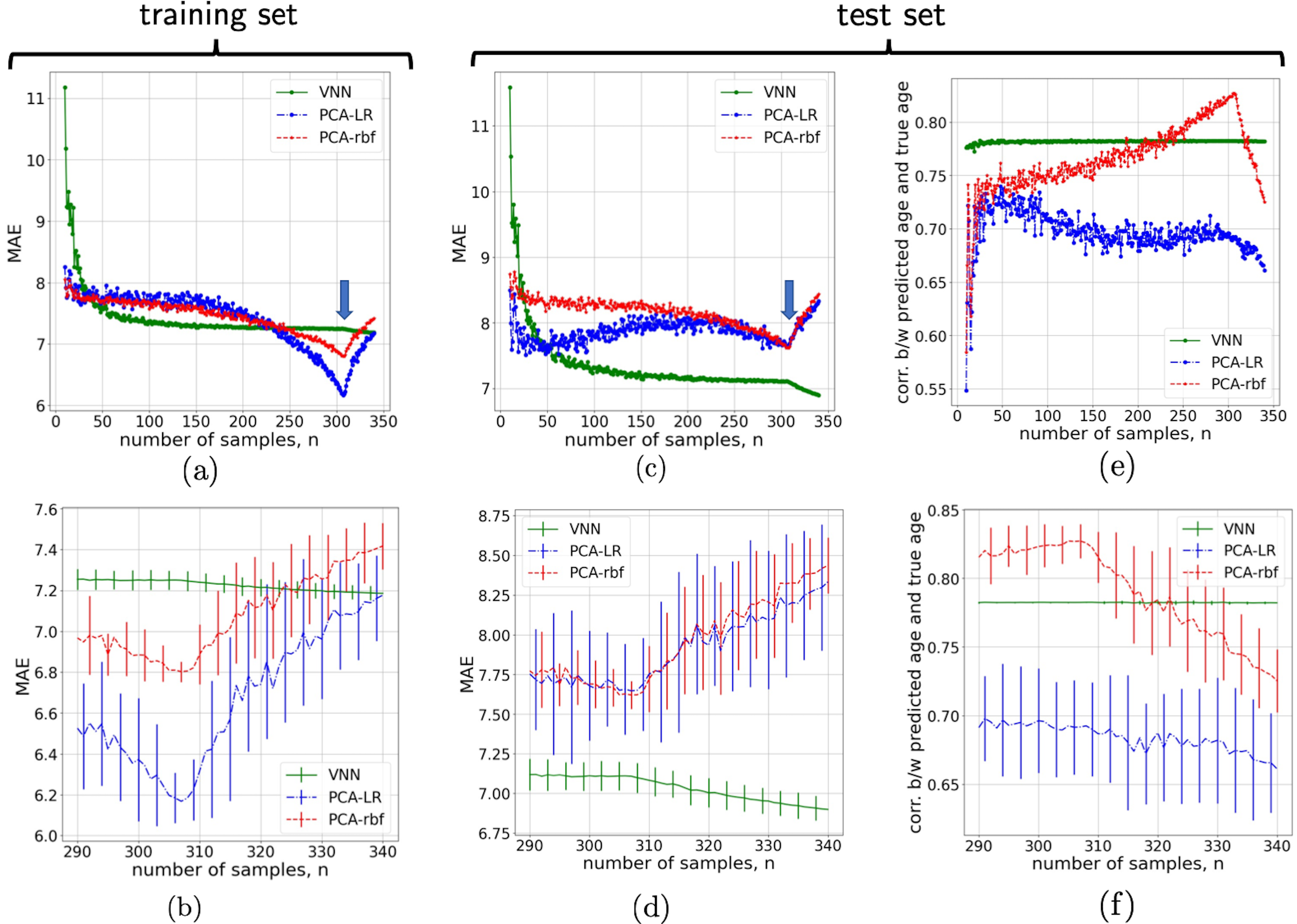}
   \caption{Stability of regression performance for VNN and PCA-regression models trained on $\hat\bC_{307}$ formed from ABC dataset. Panels (a) and (b) illustrate the performance variation on training set for the VNN and PCA-regression models when sample covariance matrix $\hat\bC_{307}$ is perturbed by addition or removal of samples ($n=307$ marked by blue arrow in (a)). Panels (c)-(f) correspond to the variation in performances in terms of MAE and correlation between predicted age and true age only on the test set. The point $n = 307$ is marked with a blue arrow in panel (c). Panels (c) and (d) illustrate the variation in mean MAE performance of the VNN and PCA-regression models trained using $\hat\bC_{307}$ on the test set ($n=307$ marked by blue arrow in (c))  with panel (d) zooming in on the range $290-341$ with error bars included. Panel (e) illustrates the variation in correlation between true age and predicted age by the VNN and PCA-regression models with panel (f) zooming in on the range $290-341$ with error bars included.  } 
   \label{perturbfig}
\end{figure}

{\bf PCA-regression}: The PCA-regression pipeline consists of two steps: i) we first identify the principle components using the eigendecomposition of the sample covariance matrix $\hat\bC_{307}$; and then, ii) to maintain consistency with VNN, transform the $273$ samples from the training set used for VNN training to fit to the corresponding age data using a regression model. Regression is implemented using sklearn package in python~\cite{pedregosa2011scikit} with `linear' and radial basis function (`rbf') kernels. PCA-regression with `rbf' kernel enables us to accommodate non-linear relationships between cortical thickness and age. PCA-regression with linear kernel in the regression model is referred to as {\sf PCA-LR} and that with `rbf' kernel in the regression model is referred to as {\sf PCA-rbf}. The optimal number of principal components in the PCA-regression pipeline are selected through a 10-fold cross-validation procedure on the training set, repeated 5 times.

For $100$ random permutations of the training set in VNN and PCA-regression experiments, we form a set of $100$ nominal models and evaluate their stability. To evaluate the stability of VNN, we replace the sample covariance matrix $\hat\bC_{307}$ with $\hat\bC_{n'}$ for $n' \in [5,341], n'\neq 307$. For PCA-regression models, we re-evaluate the principal components corresponding to $\hat\bC_{n'}$ to transform the training data while keeping the regression model learnt for PCA transformation from $\hat\bC_{307}$ fixed. Clearly, $\hat\bC_{n'}$ will be perturbed with respect to $\hat\bC_{307}$ due to finite sample size effect.



For each nominal model, we evaluate the model performances in terms of mean absolute error (MAE) and correlation between predicted age and true age for the training set and the test set.  Figures~\ref{perturbfig}~a)-b) illustrate the variation in MAE performance with perturbations to $\hat\bC_{307}$ on the training set. Figures~\ref{perturbfig}~c)-f) correspond to similar results on the test set for MAE and correlation. The mean performance in terms of MAE over $100$ nominal models is marked by a blue arrow in Fig.~\ref{perturbfig} a) for training set and in Fig.~\ref{perturbfig} c) for the test set. For PCA-regression models, we observe that both training and test performance in terms of MAE degrades significantly when the sample covariance matrix is perturbed from $\hat\bC_{307}$ by removing or adding even a small number of samples (also seen in Fig.~\ref{perturbfig}~b) and d), which are the plots focused only on the range of $290-341$ samples from Fig.~\ref{perturbfig}~a) and c), respectively, and include error bars.). In contrast, the VNN performances on both training and test sets are stable with respect to perturbations to $\hat\bC_{307}$, as suggested by our theoretical results. However, as the number of samples decrease to $n'<50$, we observe a significant decrease in VNN stability in Fig.~\ref{perturbfig}~a) and c). We also observe that the correlation between the predicted age and true age in the test set for VNN is consistently more stable than that for PCA-regression models over the entire range of samples evaluated (Fig.~\ref{perturbfig} e)). Moreover, Fig.~\ref{perturbfig} e) also demonstrates that in the range of $n'<50$ samples, there is a sharper decline in the correlation for PCA-regression models as compared to VNN despite the MAE for PCA-regression models having smaller MAE than VNN in this range. Thus, our experiments demonstrate the stability of VNNs while also illustrating that PCA-regression models may be overfit on the principal components of the sample covariance matrix used in training. Additional experiments on synthetic data and FTDC datasets also illustrate the stability of VNNs (see~Appendix~\ref{add_data}).


\iftrue

\fi

\subsection{Transferability}\label{transfer}
\begin{figure}[t]
  \centering
  \includegraphics[scale=0.4]{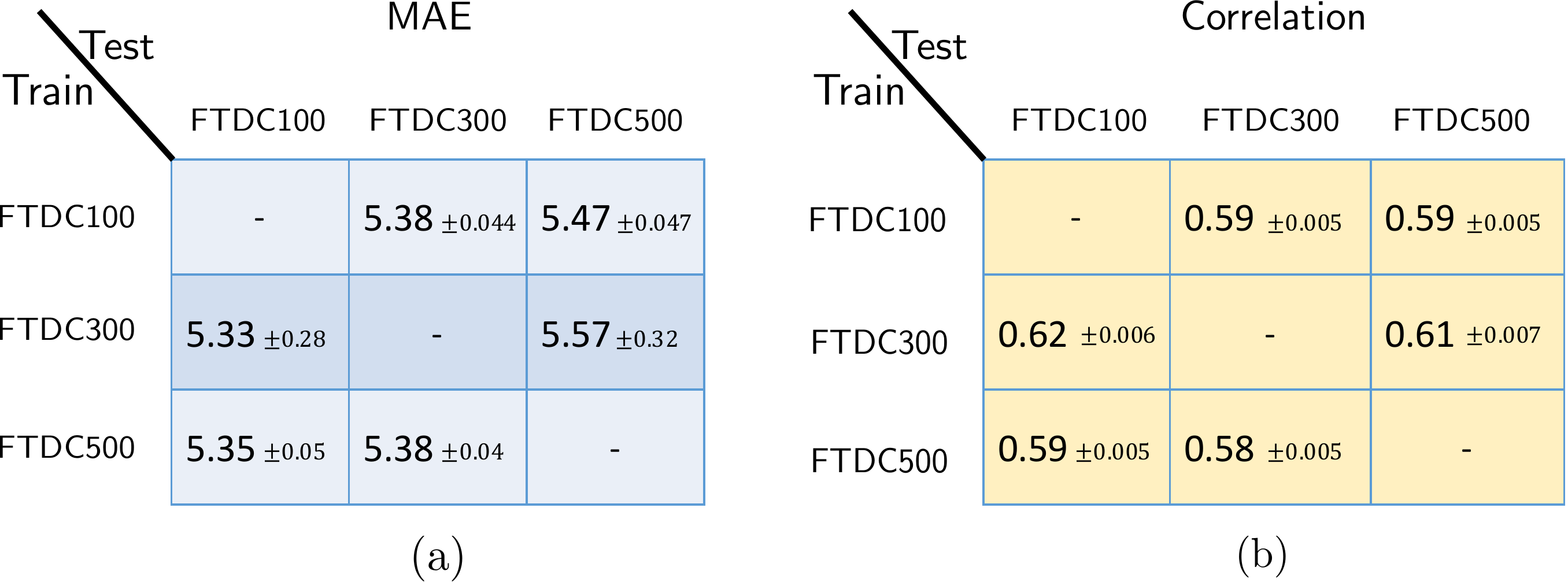}
   \caption{Transferability for VNNs across FTDC datasets.}
   \label{transfer_fig}
\end{figure}
We evaluate the transferability for VNN models trained on FTDC datasets across different resolutions. For this purpose, the VNN is trained at a specific resolution and its sample covariance matrix is replaced by the sample covariance matrix at a different resolution for testing. The training of VNNs in this context follows similar procedure as described in Section~\ref{perturb} and results in $100$ random VNN models for different permutations of the training set. For FTDC500, the VNN model consists of $1$-layer with a filter bank of $27$ $m$-dimensional outputs per layer for $m=500$ dimensions of the input data, and $2$ filter taps in each layer. The learning rate for the Adam optimizer is set to $0.008$. For FTDC300, the VNN model consists of $2$-layer architecture, with $44$ $m$-dimensional outputs per layer for $m=300$ dimensions in both layers and $2$ filter taps in each layer. The learning rate for the Adam optimizer is set to $0.0033$. For FTDC100, the VNN model consists of $2$-layer architecture, with $93$ $m$-dimensional outputs per layer for $m=100$ dimensions in both layers and $2$ filter taps in each layer. The learning rate for the optimizer is set to $0.0047$. The readout layer in each model evaluates the unweighted mean of the outputs of the final layer to form an estimate for age.

We tabulate the MAE in the matrix in Fig.~\ref{transfer_fig} a) and correlation between true age and predicted age in the matrix in Fig.~\ref{transfer_fig} b), where the row ID indicates the dataset for which VNN was trained and the column ID indicates the dataset on which the VNN was tested to evaluate transferability of VNN performance across datasets with different resolutions. For instance, the element at coordinates~$(1,2)$ in Fig.~\ref{transfer_fig}a) represents the MAE evaluated on complete FTDC300 dataset ($m=300$) for VNNs trained on FTDC100 dataset ($m=100$). The results in Fig.~\ref{transfer_fig} show that the performance of VNNs in terms of MAE and correlation between predicted age and true age can be transferred across different resolutions in the FTDC datasets. Note that this experiment is not feasible for PCA-regression models, where the principal components and the regression model would need to be re-evaluated for data from different resolution. 
\section{Conclusions}
In this paper, we have introduced coVariance neural networks (VNN) that operate on covariance matrices and whose architecture is derived from graph neural networks. We have studied the stability properties of VNNs for sample covariance matrices and shown that the stability of VNNs improves with the number of data samples $n$. Our experiments on real datasets have demonstrated that VNNs are significantly more stable than PCA-based approaches with respect to perturbations in the sample covariance matrix. Also, unlike PCA-based approaches, VNNs do not require the eigendecomposition of the sample covariance matrix.  Furthermore, on a set of multiresolution datasets, we have observed that VNN performance is also transferable across cortical thickness data collected at multiple resolutions without re-training. 

\section*{Acknowledgements}
This research was supported by National Science Foundation under the grants CCF-1750428 and CCF-1934962. The ABC dataset was provided by the Penn Alzheimer’s Disease Research Center (ADRC; NIH AG072979) at University of Pennsylvania. The MRI data for FTDC datasets were provided by the Penn Frontotemporal Degeneration Center (NIH AG066597).   Cortical thickness data were made available by Penn Image Computing and Science Lab at University of Pennsylvania.  

\newpage

\bibliographystyle{unsrtnat}
\bibliography{VNN_NeurIPS}
\newpage
\section*{Checklist}



\begin{enumerate}

\item For all authors...
\begin{enumerate}
  \item Do the main claims made in the abstract and introduction accurately reflect the paper's contributions and scope?
    \answerYes{}
  \item Did you describe the limitations of your work?
    \answerYes{} We discussed limitations of our work in Remark 2 and the discussion thereafter. 
  \item Did you discuss any potential negative societal impacts of your work?
    \answerNA{}
  \item Have you read the ethics review guidelines and ensured that your paper conforms to them?
    \answerYes{}
\end{enumerate}

\item If you are including theoretical results...
\begin{enumerate}
  \item Did you state the full set of assumptions of all theoretical results?
    \answerYes{}
        \item Did you include complete proofs of all theoretical results?
    \answerYes{}
\end{enumerate}

\item If you ran experiments...
\begin{enumerate}
  \item Did you include the code, data, and instructions needed to reproduce the main experimental results (either in the supplemental material or as a URL)?     \answerYes{} Data and code availability statement is available in Appendix~\ref{dataavail}. 

  \item Did you specify all the training details (e.g., data splits, hyperparameters, how they were chosen)?
    \answerYes{}
    The hyperparameters in the VNN architecture and learning rate of the optimizer are chosen by a hyperparameter optimization framework called Optuna~\cite{akiba2019optuna}. 
        \item Did you report error bars (e.g., with respect to the random seed after running experiments multiple times)?
    \answerYes{}
        \item Did you include the total amount of compute and the type of resources used (e.g., type of GPUs, internal cluster, or cloud provider)?
    
    The experiments were run on a $12$GB Nvidia GeForce RTX 3060 GPU.
\end{enumerate}

\item If you are using existing assets (e.g., code, data, models) or curating/releasing new assets...
\begin{enumerate}
  \item If your work uses existing assets, did you cite the creators?
    \answerYes{} 
  \item Did you mention the license of the assets?
    \answerNA{}
  \item Did you include any new assets either in the supplemental material or as a URL? \answerNA{}
  \item Did you discuss whether and how consent was obtained from people whose data you're using/curating?
    \answerYes{}
 All participants in the ABC and FTDC datasets took part in an informed consent procedure approved by an Institutional Review Board convened at University of Pennsylvania.
  \item Did you discuss whether the data you are using/curating contains personally identifiable information or offensive content?
    \answerNA{} 
    The data contains no personally identifiable information or offensive content.
    
\end{enumerate}

\item If you used crowdsourcing or conducted research with human subjects...
\begin{enumerate}
  \item Did you include the full text of instructions given to participants and screenshots, if applicable?
    \answerNA{}
  \item Did you describe any potential participant risks, with links to Institutional Review Board (IRB) approvals, if applicable? \answerNA{}
  \item Did you include the estimated hourly wage paid to participants and the total amount spent on participant compensation?
   
   All participants in the ABC and FTDC datasets received a monetary honorarium (of $\$25-\$40$ per participant for participation).
\end{enumerate}

\end{enumerate}
\newpage
\appendix
\section{Data and Code Availability}\label{dataavail}
Data for experiments on ABC and FTDC datasets may be requested through \url{https://www.pennbindlab.com/data-sharing} and upon review by the University of Pennsylvania Neurodegenerative Data Sharing Committee, access will be granted upon reasonable request. Data and code for synthetic experiments in Appendix D are available at \url{https://github.com/pennbindlab/VNN}. The implementations of the experiments on ABC and FTDC datasets are similar.
\section{Overview of PCA}\label{pca_ov}
The advantages of using PCA, its role in dimension reduction, and limitations are well documented~\cite{shlens2014tutorial}. 
 PCA is an orthogonal, linear transformation of the observed dataset~$\hat\bx_n$ via a change of basis, such that, the transformed data has minimal redundancy while the hidden structure or patterns of interest in the raw data are preserved for further statistical analyses. In practice, the implementation of PCA is given by 
\begin{align}
    \hat\by_n = \bP\hat\bx_n\;,
\end{align}
where the rows of $\bP \in \mR^{m\times m}$ represent the new basis on which the data $\hat\bx_n$ is projected and therefore, form the designed principal components in PCA. We next briefly discuss the construction of the transformation $\bP$ and the intuition behind it. Note that in the diagonal elements of the sample covariance matrix $\hat \bC_n$  represent the individual variances of the constituent $m$ features and its off-diagonal elements characterize the redundancy in the data. A desired objective of PCA is to de-correlate the dataset $\hat\bx_n$ by removing any second order dependencies. To achieve this, a linear transformation $\bP$ is designed, such that, the covariance of the transformed data $\by_n$, given by $\hat \bC_{\by} $ is a diagonal matrix. It can readily be verified that the eigenvectors of the sample covariance matrix $\hat \bC_n$ satisfy this property. The eigendecomposition of $ \hat \bC_n$ is formalized as
\begin{align}
    \hat \bC_n = \bU \bW \bU^{\sf T}\;,
\end{align}
where $\cdot^{\sf T}$ refers to the Hermitian operator, $\bU$ is the matrix constituted by orthonormal eigenvectors of $\hat \bC_n$ and $\bW= {\sf diag}(w_1,\dots,w_m)$ is the diagonal matrix of eigenvalues of $\hat \bC_n$, such that, $w_1\geq w_2\cdots\geq w_m$. Therefore, for transformation $\bP = \bU^{\sf T}$, such that, $\hat\by_n = \bU^{\sf T}\hat\bx_n$ we have
\begin{align}
    \hat \bC_{\by} = \frac{1}{n} \bU^{\sf T} \bx_n\bx_n^{\sf T} \bU =  \bU^{\sf T} (\bU \bW \bU^{\sf T}) \bU = \bW\;.
\end{align}
Hence, in practical implementation of PCA, the eigenvectors of the sample covariance matrix form the principal components that characterize the linear, orthogonal transformation of data $\hat\bx_n$. A common application of PCA is to extract the principal components that explain the most variance in $\hat\bx_n$ (thus, resulting in dimension reduction while preserving the most relevant information under certain assumptions), followed by designing statistical models for inference tasks such as prediction, classification etc~\cite{jolliffe2016principal}. 
\section{Stability of coVariance Filters}\label{stab_pf}
We start by characterizing the perturbation of sample covariance matrix $\hat\bC_n$
 with respect to $\bC$ in Lemma~\ref{lm1}.  For convenience, we use the notation $\hat\bC$ for $\hat\bC_n$.  To this end, we define
 \begin{align}
     \bE \dff \hat\bC - \bC\;,
 \end{align}
and $\bI_m$ as an $m\times m$ identity matrix.
\begin{lemma}\label{lm1}
Consider an ensemble covariance matrix $\bC$ with the eigendecomposition in~\eqref{eigenC} and a sample covariance matrix $\hat\bC$ with the eigendecomposition in~\eqref{sample_cov}. For any eigenvalue $\lambda_i > 0$ of $\bC$, the perturbation $\bE$ satisfies
\begin{align}
    \bE \bv_i = \beta_i\delta\bv_i + \delta\lambda_i \bv_i + ( \delta\lambda_i\bI_m -\bE) \delta\bv_i
\end{align}
where 
\begin{align}
    \beta_i \dff  (\lambda_i \bI_{m} - \bC),\quad \delta\bv_i \dff \bu_i-\bv_i,\quad \delta\lambda_i \dff w_i-\lambda_i\;.
\end{align}

\end{lemma}
\begin{proof}

Note that from the definition of eigenvectors and eigenvalues, we have
\begin{align}\label{egic}
   \hat\bC \bu_i = w_i\bu_i\;.
\end{align}
We can rewrite~\eqref{egic} in terms of perturbations with respect to the ensemble covariance matrix $\bC$ and the outputs of its eigendecomposition as follows:
\begin{align}\label{prt}
   (\hat\bC -\bC)  (\bv_i + \delta\bv_i) + \bC (\bv_i + \delta\bv_i) = (\lambda_i + \delta\lambda_i)(\bv_i + \delta\bv_i)\;,
\end{align}
where we have used $w_i = \lambda_i + \delta\lambda_i$ and $\bu_i = \bv_i + \delta\bv_i$. Using the fact that $\bC \bv_i = \lambda_i\bv_i$ and rearranging the terms in~\eqref{prt}, we have
\begin{align}\label{prt2}
    (\hat\bC -\bC) \bv_i = (\lambda_i \bI_{m} - \bC)\delta\bv_i + \delta\lambda_i  (\bv_i + \delta\bv_i) - (\hat\bC -\bC) \delta\bv_i\;.
\end{align}
By setting $\bE = \hat\bC -\bC$ and $\beta_i = \lambda_i \bI_{m} - \bC$, we can rewrite~\eqref{prt2} as
\begin{align}
    \bE\bv_i = \beta_i\delta\bv_i + \delta\lambda_i \bv_i + (\delta\lambda_i\bI_m - \bE) \delta\bv_i\;.
\end{align}
\end{proof}
Next, we leverage Lemma~\ref{lm1} to complete the proof of Theorem~\ref{filterstab}.
\section*{Proof of Theorem~\ref{filterstab}}
\begin{proof}
To start with, we note that the coVariance filters with respect to $\hat\bC$ and $\bC$ are given by
\begin{align}
    \bH(\hat\bC) = \sum\limits_{k=0}^{m}h_k\hat\bC^k \quad \text{and}\quad  \bH(\bC) = \sum\limits_{k=0}^{m}h_k\bC^k\;.
\end{align}
We aim to study the stability of the coVariance filters by analyzing the difference between $\bH(\hat\bC)$ and $\bH(\bC)$. For this purpose, we next establish the first order approximation for $\hat\bC^k $ in terms of $\bC$ and $\bE$. Using $\hat\bC = \bC + \bE$, the first order approximation of $\hat\bC^k$ is given by
\begin{align}\label{pf2}
    (\bC + \bE)^k = \bC^k + \sum\limits_{r=0}^{m} \bC^r \bE \bC^{k-r-1} + \tilde\bE\;,
\end{align}
where $\tilde \bE$ satisfies $\|\tilde \bE\| \leq \sum\limits_{r=2}^k {k\choose r}\|\bE\|^r\|\bC\|^{k-r}$. Using~\eqref{pf2}, we have
\begin{align}
    \bH(\hat\bC) - \bH(\bC) &= \sum\limits_{k=0}^{m} h_k [(\bC + \bE)^k - \bC^k]\;,\\
    &= \sum\limits_{k=0}^{m} h_k \sum\limits_{r=0}^{k-1} \bC^r \bE \bC^{k-r-1} + \tilde \bD\;,\label{pf3}
\end{align}
where $\tilde\bD$ satisfies $\|\tilde\bD\|^2 = {\cal O}(\|\bE\|^2)$~\cite{gama2020stability}. The focus of our subsequent analysis will be the first term in~\eqref{pf3}. For a random data sample $\bx = [x_1,\dots,x_m]^{\sf T}$, such that, $\|\bx\| < R$, for some $R>0$ and $\bx \in \mR^{m\times 1}$,  its VFT with respect to $\bC$ is given by $\tilde\bx = \bV^{\sf T} \bx$, where $\tilde\bx = [\tilde x_1,\dots,\tilde x_m]^{\sf T}$. The relationship $\tilde \bx $ and $\bx$ can be expressed as
\begin{align}\label{pf5}
   \bx = \sum\limits_{i=1}^m \tilde x_i \bv_i\;.
\end{align}
Multiplying both sides in~\eqref{pf3} by $\bx$ and by leveraging~\eqref{pf5}, we get
\begin{align}
    [\bH(\hat\bC) - \bH(\bC) ]\bx &= \sum\limits_{k=0}^{m} h_k \sum\limits_{r=0}^{k-1} \bC^r \bE \bC^{k-r-1} \bx + \tilde \bD\bx\;,\\
    &= \sum\limits_{i=1}^m \tilde x_i \sum\limits_{k=0}^{m} h_k \sum\limits_{r=0}^{k=1} \bC^r \bE \bC^{k-r-1}  \bv_i + \tilde \bD\bx\;,\label{t1}\\
    &=  \sum\limits_{i=1}^m \tilde x_i \sum\limits_{k=0}^{m} h_k \sum\limits_{r=0}^{k-1} \bC^r\lambda_i^{k-r-1} \bE\bv_i + \tilde \bD\bx\;,\label{t2}
\end{align}
where we have used $\bC \bv_i = \lambda_i\bv_i$ in the transition from~\eqref{t1} to~\eqref{t2}. 
We focus only on the first term in~\eqref{t2} and leverage the result from Lemma~\ref{lm1} that expands $\bE \bv_i$ to get
\begin{align}\label{pf6}
    \sum\limits_{i=1}^m \tilde x_i \sum\limits_{k=0}^{m} h_k \sum\limits_{r=0}^{k-1} \bC^r\lambda_i^{k-r-1} \bE\bv_i &= \underbrace{ \sum\limits_{i=1}^m \tilde x_i \sum\limits_{k=0}^{m} h_k \sum\limits_{r=0}^{k-1} \bC^r\lambda_i^{k-r-1} \beta_i\delta\bv_i}_\text{Term 1} \nonumber\\
    & \enskip + \underbrace{ \sum\limits_{i=1}^m \tilde x_i \sum\limits_{k=0}^{m} h_k \sum\limits_{r=0}^{k-1} \bC^r\lambda_i^{k-r-1} \delta\lambda_i\bv_i }_\text{Term 2}\nonumber\\
    &\enskip +  \underbrace{ \sum\limits_{i=1}^m \tilde x_i \sum\limits_{k=0}^{m} h_k \sum\limits_{r=0}^{k-1} \bC^r\lambda_i^{k-r-1} ( \delta\lambda_i\bI_m -\bE)\delta\bv_i }_\text{Term 3}\;.
\end{align}

Next, we analyze term 1, term 2, and term 3 in~\eqref{pf6} separately.

\textbf{Analysis of Term 1 in~\eqref{pf6}}: In the analysis of term 1, we start by noting that
\begin{align}
    \beta_i &= \lambda_i\bI_m - \bC \;,\\
    &= \sum\limits_{j = 1}^m (\lambda_i - \lambda_j)\bv_j\bv_j^{\sf T}\;,\label{t11}\\
    &=  \bV (\lambda_i\bI_m - \Lambda)\bV^{\sf T}\;.\label{t12}
\end{align}
Using~\eqref{t12} and $\delta\bv_i = \bu_i - \bv_i$ in term 1 in~\eqref{pf6}, we have
\begin{align}
     \sum\limits_{i=1}^m \tilde x_i \sum\limits_{k=0}^{m} h_k \sum\limits_{r=0}^{k-1} \bC^r\lambda_i^{k-r-1} \bV (\lambda_i\bI_m - \Lambda)\bV^{\sf T} (\bu_i-\bv_i)\;.\label{t13}
\end{align}
Using $\bC^r = \bV \Lambda^r \bV^{\sf T}$ in~\eqref{t13}, term 1 in~\eqref{pf6} is equivalent to
\begin{align}
    &\sum\limits_{i=1}^m \tilde x_i \sum\limits_{k=0}^{m} h_k \sum\limits_{r=0}^{k-1} \lambda_i^{k-r-1} \bV \Lambda^r (\lambda_i\bI_m - \Lambda)\bV^{\sf T} (\bu_i-\bv_i)\;,\\
    &= \sum\limits_{i=1}^m \tilde x_i \bV \bL_i \bV^{\sf T} (\bu_i - \bv_i)\;,\label{pf7}
\end{align}
where $\bL_i$ is a diagonal matrix whose $j$-th element is given by
\begin{align}
    [\bL_i]_j &= \sum\limits_{k=0}^{m} h_k \sum\limits_{r=0}^{k-1} (\lambda_i - \lambda_j) \lambda_i^{k-r-1} \lambda_j^r\;,\\
    &= \sum\limits_{k=0}^{m} h_k (\lambda_i - \lambda_j)\frac{\lambda_i^k - \lambda_j^k}{\lambda_i - \lambda_j}\;,\\
    & = \sum\limits_{k=0}^{m}h_k \lambda_i^k - \sum\limits_{k=0}^{m}h_k \lambda_j^k\;,\\
    &= h(\lambda_i) - h(\lambda_j)\;,\label{t15}
\end{align}
where $ h(\lambda_i)$ is the frequency response of the coVariance filter and is defined in~\eqref{vvf}. Therefore, we have $\bL_i = {\sf diag}([h(\lambda_i) - h(\lambda_j)]_j)$. Next, in~\eqref{pf7}, we note that 
\begin{align}\label{t14}
    \bV^{\sf T}(\bu_i - \bv_i) = [\bv_1^{\sf T}(\bu_i -\bv_i) , \cdots, \bv_m^{\sf T}(\bu_i - \bv_i)]^{\sf T}\;.
\end{align}
Using~\eqref{t14} and~\eqref{t15} in~\eqref{pf7} and $\bv_j^{\sf T}\bv_i = 0, \forall j\neq i$, we deduce that the term 1 in~\eqref{pf6} is equivalent to
\begin{align}\label{t1f}
    \sum\limits_{i=1}^m \tilde x_i \bV \bL_i \bV^{\sf T} (\bu_i - \bv_i) = \sum\limits_{i=1}^m \tilde x_i \bV \bJ_i\;,
\end{align}
where the $j$-th element of $\bJ_i$ is given by
\begin{align}
    [\bJ_i]_j = \begin{cases}
    &0\;, \quad \text{if }\enskip j = i\;,\\
    &(h(\lambda_i) - h(\lambda_j) )\bv_j^{\sf T} \bu_i\;, \enskip \text{otherwise}
    \end{cases}\;.
\end{align}
For the stability analysis, we are interested in the norm of term 1. Therefore, by noting the equivalence between the term 1 in~\eqref{pf6} and~\eqref{t1f}, after taking the uniform norm, we have
\begin{align}
    \left\lVert\sum\limits_{i=1}^m \tilde x_i \sum\limits_{k=0}^{m} h_k \sum\limits_{r=0}^{k-1} \bC^r\lambda_i^{k-r-1} \beta_i\delta\bv_i\right\rVert &=\left\lVert\sum\limits_{i=1}^m \tilde x_i \bV \bJ_i\right\rVert \;,\\
    & \leq \sum\limits_{i=1}^m |\tilde x_i| \max_{j, i\neq j} \|h(\lambda_i)- h(\lambda_j)\| \|\bv_j^{\sf T} \bu_i\|\;.
\end{align}
Note that $\bv_j^{\sf T} \bu_i$ is the inner product between the eigenvector $\bv_j$ of ensemble covariance matrix $\bC$ and the eigenvector $\bu_i$ of the sample covariance matrix $\hat\bC$. The bounds on $\bv_j^{\sf T} \bu_i$ in terms of the number of data samples $n$ have been studied in the existing literature. Here, we leverage the result from~\cite[Theorem 4.1]{loukas2017close} to conclude that if ${\sf sgn}(\lambda_j-\lambda_i)2 w_j > {\sf sgn}(\lambda_j-\lambda_i)(\lambda_j-\lambda_i)$ for $\lambda_i\neq \lambda_j$, the condition
\begin{align}\label{term1r}
     \left\lVert\sum\limits_{i=1}^m \tilde x_i \sum\limits_{k=0}^{m} h_k \sum\limits_{r=0}^{k-1} \bC^r\lambda_i^{k-r-1} \beta_i\delta\bv_i\right\rVert \leq \sum\limits_{i=1}^m |\tilde x_i|\max_{j, i\neq j} |h(\lambda_i)- h(\lambda_j)| \frac{2k_i}{n^{1/2  -\varepsilon}|\lambda_i - \lambda_j|}\;,
\end{align}
is true with probability at least $\left(1 - \frac{1}{n^{2\varepsilon}}\right)$ for some $\varepsilon \in (0,1/2]$, where $k_i = \Big(\mE[\|\bX\bX^{\sf T} \bv_i\|_2^2] - \lambda_i^2\Big)^{\frac{1}{2}}$. Furthermore, we note that the condition~$ {\sf sgn}(\lambda_j-\lambda_i)2 w_j > {\sf sgn}(\lambda_j-\lambda_i)(\lambda_j-\lambda_i)$ is satisfied with probability at least $1 - \frac{2k_i^2}{|\lambda_i - \lambda_j|}$~\cite[Corollary 4.2]{loukas2017close}, which via a union bound and first order approximation from Taylor series implies that~\eqref{term1r} is true with probability at least $1 - \frac{1}{n^{2\varepsilon}} - \frac{2\kappa m}{n}$ for $\kappa$ defined in~\eqref{addefs}. Therefore, for a coVariance filter with the property
\begin{align}\label{fltrprop}
    \max_{i,j \in \{1,\dots,m\}, i\neq j} \frac{|h(\lambda_i)-h(\lambda_j)|}{|\lambda_i-\lambda_j|}\leq \frac{M}{k_i}\;,
\end{align}
for some real constant $M > 0$, the condition in~\eqref{term1r} is equivalent to
\begin{align}
     \left\lVert\sum\limits_{i=1}^m \tilde x_i \sum\limits_{k=0}^{m} h_k \sum\limits_{r=0}^{k-1} \bC^r\lambda_i^{k-r-1} \beta_i\delta\bv_i\right\rVert \leq \frac{2 M}{n^{\frac{1}{2}-\varepsilon}}\sum\limits_{i=1}^m |\tilde x_i|\;,
\end{align}
which holds with probability at least  $1 - \frac{1}{n^{2\varepsilon}} - \frac{2\kappa m}{n}$. Furthermore, note that $\sum\limits_{i=1}^m |\tilde x_i| \leq \sqrt{m}\|\bx\|_2$. If the random sample $\bx$ satisfies $\|\bx\|_2 \leq R$, then we have
\begin{align}\label{final_pf1}
    \mP\left( \left\lVert\sum\limits_{i=1}^m \tilde x_i \sum\limits_{k=0}^{m} h_k \sum\limits_{r=0}^{k-1} \bC^r\lambda_i^{k-r-1} \beta_i\delta\bv_i\right\rVert \leq \frac{2}{n^{\frac{1}{2}-\varepsilon}}\sqrt{m} M R\right)  \geq 1-\frac{1}{n^{2\varepsilon}} - \frac{2\kappa m}{n}\;,
\end{align}
for any $\varepsilon\in (0,1/2]$.

\textbf{Analysis of Term 2 in~\eqref{pf6}}:
Using $\bC \bv_i = \lambda_i\bv_i$, we note that term 2 in~\eqref{pf6} is equivalent to
\begin{align}
    \sum\limits_{i=1}^m \tilde x_i \sum\limits_{k=0}^{m} h_k \sum\limits_{r=0}^{k-1} \bC^r\lambda_i^{k-r-1} \delta\lambda_i\bv_i &= \sum\limits_{i=1}^m \tilde x_i \sum\limits_{k=0}^{m} h_k \sum\limits_{r=0}^{k-1} \lambda_i^{k-1} \delta\lambda_i\bv_i \;,\\
    &=  \sum\limits_{i=1}^m \tilde x_i \sum\limits_{k=0}^{m}k h_k \lambda_i^{k-1} \delta\lambda_i\bv_i \;,\\
    &=  \sum\limits_{i=1}^m \tilde x_i h'(\lambda_i) \delta\lambda_i\bv_i \;.\label{simpl}
\end{align}
Next, using Weyl's theorem~\cite[Theorem 8.1.6]{golub2013matrix}, we note that  $\|\bE\| \leq \alpha$ implies that $|\delta\lambda_i| \leq \alpha$ for any $\alpha>0$. For a random instance $\bx$ of random vector $\bX$ whose  probability distribution is supported within a bounded region w.l.o.g, such that, $\|\bx\|\leq 1$, we have
\begin{align}\label{alpha}
    \mP\left(\bE \leq B\left(\frac{\|\bC\|\sqrt{\log m + u}}{\sqrt{n}} + \frac{(1 + \|\bC\|)(\log m+u)}{n}\right)\right) \geq 1 - 2^{-u}\;,
\end{align}
for some constant $B > 0$ and $u > 0$. The result in~\eqref{alpha} follows directly from~\cite[Theorem 5.6.1]{vershynin2018high}. Therefore, using~\eqref{simpl}, we have
\begin{align}
    \left\lVert\sum\limits_{i=1}^m \tilde x_i \sum\limits_{k=0}^{m} h_k \sum\limits_{r=0}^{k-1} \bC^r\lambda_i^{k-r-1} \delta\lambda_i\bv_i\right\rVert &\leq \sum\limits_{i=1}^m |\tilde x_i| |h'(\lambda_i)| |\delta\lambda_i| \|\bv_i\|\;.
\end{align}
Using~\eqref{alpha}, $|h'(\lambda_i)|\leq M/k_{\sf min}$ (where $k_{\sf min} = \min_{i\in \{1,\dots,m\}, \lambda_i>0}k_i$) from~\eqref{fltrprop}, and $\|\bv_i\| = 1$, we have
\begin{align}\label{final_pf2}
    &\mP\left( \left\lVert\sum\limits_{i=1}^m \tilde x_i \sum\limits_{k=0}^{m} h_k \sum\limits_{r=0}^{k-1} \bC^r\lambda_i^{k-r-1} \delta\lambda_i\bv_i\right\rVert \right.\nonumber\\
    &\quad \left.\leq \frac{A}{k_{\sf min}}\sqrt{m}M \left(\frac{\|\bC\|\sqrt{\log m + u}}{\sqrt{n}} + \frac{(1 + \|\bC\|)(\log m+u)}{n}\right)\right) \geq 1 - 2^{-u}\;,
\end{align}
for some constant $A>0$ and $u>0$.

\textbf{Analysis of Term 3 in~\eqref{pf6}}: We remark that the term 3 in~\eqref{pf6} consists of second order error terms that diminish faster with the number of samples $n$ as compared to term 1 and term 2. To illustrate this, we note two facts. First, from~\eqref{alpha} and Weyl's theorem, we note that $\| \delta\lambda_i\bI_m -\bE\| \leq 2\|\bE\|$ and $\|\bE\| \simeq {\cal O}(1/\sqrt{n})$ with high probability. Secondly, from the existing literature~\cite{taylor2019ensemble,mestre2008asymptotic}, we note that if $\hat\bC$ follows the Wishart distribution, we have
\begin{align}
    \mE[\|\delta \bv_i\|] = 0
\end{align}
and
\begin{align}
    \mE[n\|\delta\bv_i\|^2] = \sum\limits_{j\neq i} \frac{\lambda_i\lambda_j}{(\lambda_i-\lambda_j)^2}\;.
\end{align}
Let $\sigma_{\bv_i}^2 \triangleq\sum\limits_{j\neq i} \frac{\lambda_i\lambda_j}{(\lambda_i-\lambda_j)^2}$. Using Chebyshev's inequality, we note that
\begin{align}\label{cheb}
    \mP\left(\sqrt{n}\|\delta\bv_i\| \geq \gamma \sigma_{\bv_i}\right) \leq \frac{1}{\gamma^2}\;,
\end{align}
for any constant $\gamma > 1$. Furthermore,~\eqref{cheb} is equivalent to
\begin{align}
     \mP\left(\|\delta\bv_i\| \leq \frac{\gamma}{\sqrt{n}} \sigma_{\bv_i}\right) \geq 1 - \frac{1}{\gamma^2}\;,
\end{align}
which implies that $\|\delta\bv_i\| = {\cal O}(1/\sqrt{n})$ with high probability. Therefore, the second order error term $( \delta \lambda_i\bI_m -\bE)\delta\bv_i$ scales as ${\cal O}(1/n)$, which diminishes faster with $n$ as compared to terms 1 and 2, that individually scale as ${\cal O}(1/n^{1/2 - \varepsilon})$  for $\varepsilon \in (0,1/2]$ and ${\cal O}(1/\sqrt{n})$, respectively. 

The proof of Theorem~\ref{filterstab} is completed by noting that the condition on $\|[\bH(\hat\bC)-\bH(\bC)]\bx\|$ reduces to the condition on operator norm  $\|[\bH(\hat\bC)-\bH(\bC)]\|$ for any $\|\bx\|\leq 1$ and that the terms scaling at $1/\sqrt{n}$ or slower in~\eqref{final_pf1} and~\eqref{final_pf2} dominate the scaling behavior of the upper bound on $\|[\bH(\hat\bC)-\bH(\bC)]\|$. 
\end{proof}

\section{VNN architecture}\label{vnn_arch}

\begin{figure}[H]
  \centering
  \includegraphics[scale=0.25]{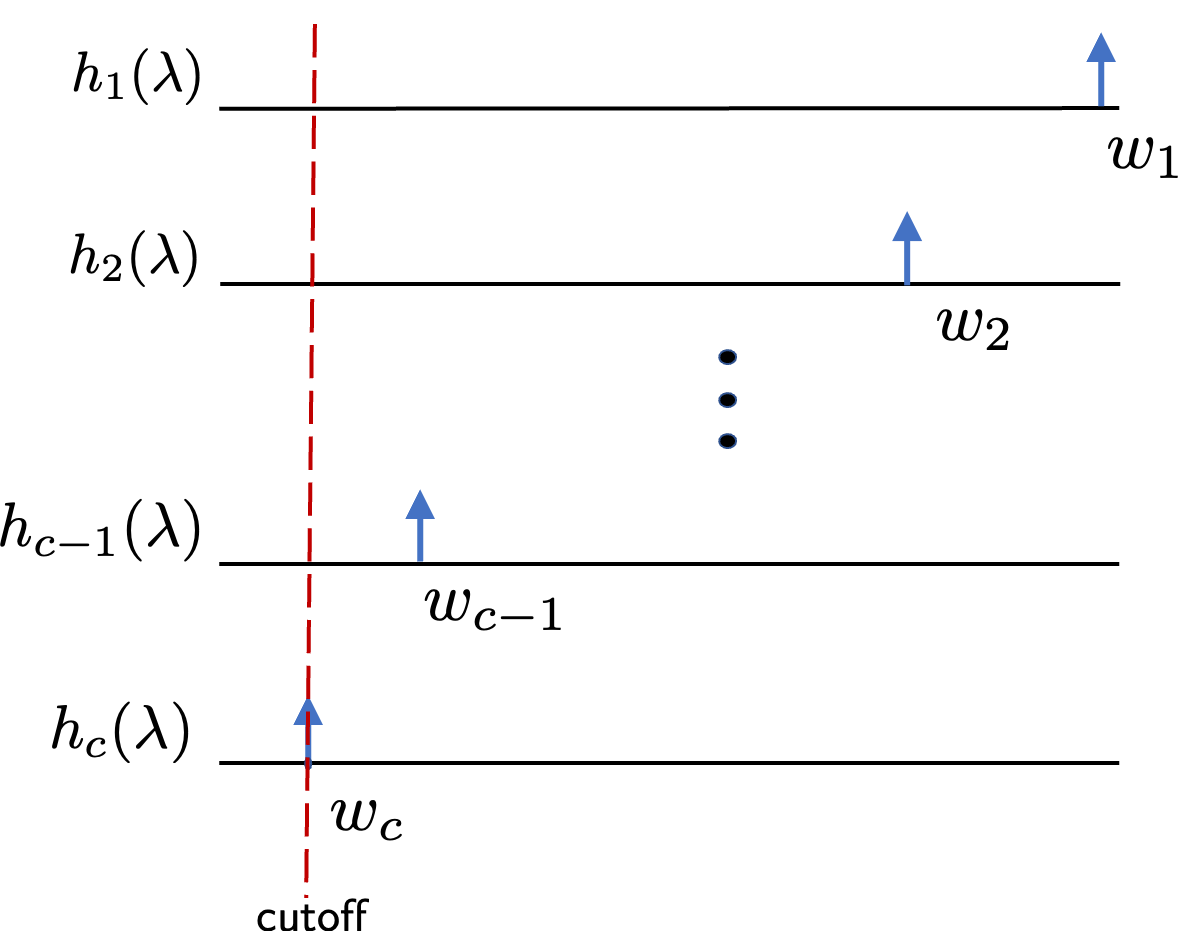}
   \caption{Filter response for different coVariance filters sufficient to implement PCA transformation that includes $c$ largest eigenvalues of the covariance matrix $\hat\bC_n$ according to Theorem~\ref{pca_cf}.}
   \label{pcafilters}
\end{figure}

\begin{figure}[H]
  \centering
  \includegraphics[scale=0.75]{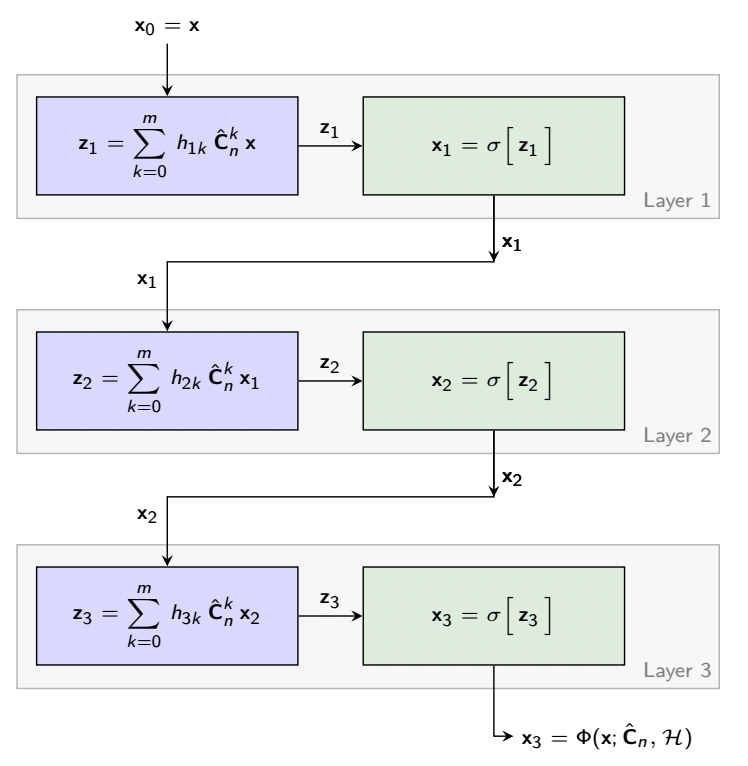}
   \caption{A 3-layer VNN architecture.}
   \label{vnnfig}
\end{figure}

\section{Additional Experiments and Discussions}\label{add_data}
In this section, we provide additional discussions on the age prediction task using cortical thickness and related experiments on the synthetic data. In Section~\ref{cort_data}, we briefly discuss the motivation behind studying age prediction and PCA-based statistical analysis in this context. In Section~\ref{dataacq}, we provide additional details on cortical thickness data acquisition. In Section~\ref{cmpr}, we report the results for stability analysis of VNNs and PCA-regression models for FTDC100 ($m=100$) and FTDC300 ($m=300$) datasets. In Section~\ref{cmpr_synth}, we study the stability of VNNs on two simulated settings that include non-linear and linear data models. In Section~\ref{transfer_figs}, we include additional figures that supplement the VNN transferability results in Fig.~\ref{transfer_fig}.

\subsection{Predicting Brain Age with Cortical Thickness Data}\label{cort_data}
Brain age estimation using various modalities of magnetic resonance imaging (MRI) is an active area of research~\cite{sajedi2019age,valizadeh2017age,aycheh2018biological}. A promising application of brain age prediction is early detection of neurodegenerative diseases (such as Alzheimer's, Huntingson's disease) which may manifest themselves as error in age prediction in pathological contexts by machine learning models trained on healthy subjects. Cortex anatomical measures extracted from structural MRI scans have shown promising results in age prediction in existing studies~\cite{valizadeh2017age,aycheh2018biological}. Cortical measures across the brain usually have high collinearity and therefore, the age prediction pipeline consisting of dimension reduction on cortical features and regression models is commonly employed. The utility of dimension reduction is apparent in ABC and FTDC datasets used in this paper as well from Fig.~\ref{abc_eig} and Fig.~\ref{ftdc_eig}, where we observe that the eigenvalue distribution for ABC and FTDC300 datasets is skewed and therefore, PCA based dimensionality reduction is well-motivated. 
\begin{figure*}[t]
  \centering
  \begin{subfigure}{0.45\textwidth}
  \includegraphics[scale=0.25]{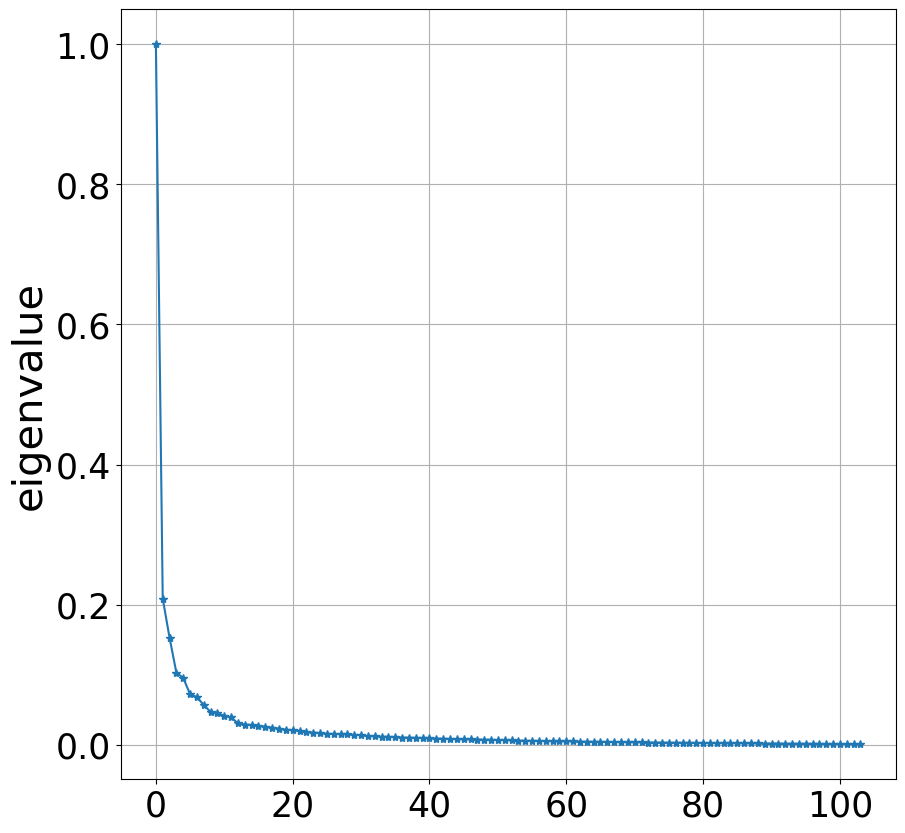}
   \caption{Eigenvalue distribution for ABC dataset.}
   \label{abc_eig}
 \end{subfigure}
   \hfill
   \begin{subfigure}{0.45\textwidth}
  \centering
  \includegraphics[scale=0.25]{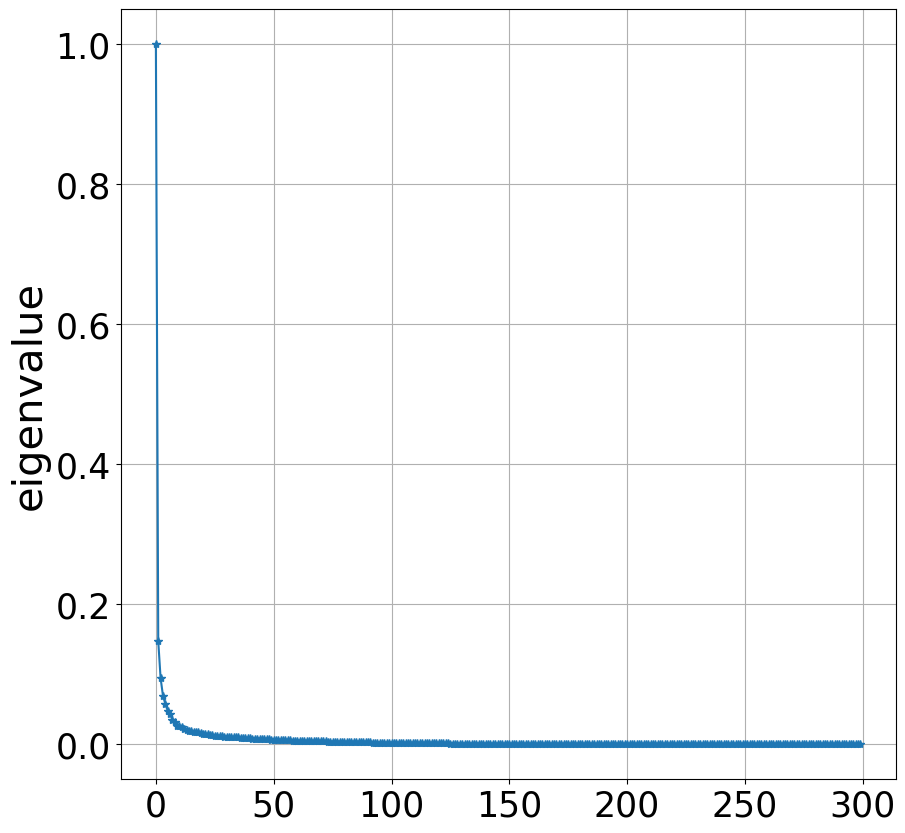}
   \caption{Eigenvalue distribution for FTDC300 dataset.}
   \label{ftdc_eig}
\end{subfigure}
\caption{Eigenvalue distributions.}
\end{figure*}

The studies in~\cite{valizadeh2017age} and~\cite{aycheh2018biological} report age prediction results on different datasets using various approaches, such as PCA-based regression (similar to the PCA-regression models evaluated in this paper), ridge regression, lasso based regression, and fully connected neural networks. All these approaches show comparable performances in predicting age using cortical anatomical measures~\cite{valizadeh2017age}. The objective of our experiments is to demonstrate the properties of VNNs, such as stability and transferability, against other covariance matrix driven statistical analyses. In this context, PCA-based regression models are the natural baselines for comparison against VNNs. 

\subsection{Cortical Thickness Data Acquisition}\label{dataacq}
Cortical thickness measures in regions of interest (ROIs) were derived using $0.8-1$ mm isotropic T1-weighted MRI. The complete pipeline for cortical thickness extraction for ABC data is similar to that in~\cite{phillips2019longitudinal}.

\subsection{Stability of VNNs on FTDC100 and FTDC300 datasets}\label{cmpr}
In this subsection, we study the stability of VNNs for FTDC100 and FTDC300 datasets in a similar fashion as for ABC dataset in Section~\ref{vnn_stab}. We split each dataset into a $90/10$ train/test split, such that, we have $153$ samples in the training set and $17$ samples in the test set. The nominal VNN and PCA-regression models for FTDC100 and FTDC300 are trained on their respective sample covariance matrices $\hat\bC_{153}$ derived from cortical thickness data. The architecture and hyperparamaters for VNN training for FTDC100 and FTDC300 are same as that reported in Section~\ref{transfer}. Figure~\ref{FTDC_sta} a) and b) show the variance in MAE over training and test sets for nominal models based on VNN and PCA-regression with respect to $\hat\bC_{153}$ for FTDC100 dataset. Figure~\ref{FTDC_sta} c) and d) show the variance in MAE over training and test sets for nominal models based on VNN and PCA-regression with respect to $\hat\bC_{153}$ for FTDC300 dataset. For both datasets, we observe that VNNs are stable with respect to the perturbations in the sample covariance matrix. Also, note that our theoretical results in Theorem~\ref{perturb} and Theorem~\ref{vnn_stab} suggest stability of coVariance filters and VNNs for $n>m$. Our observations from the experiments on FTDC300 dataset show evidence that VNNs retain stability even when this condition is violated. Significant randomness is induced into the performance of ${\sf PCA-LR}$ when the principal components are perturbed due to perturbations in $\hat\bC_{153}$ for both datasets. However, in Fig.~\ref{FTDC_sta}, we also observe that {\sf PCA-rbf} retains stability in performance when $\hat\bC_{153}$ is replaced by $\hat\bC_{n'}$ for $n'\in [65,152]$.

\begin{figure}[t]
  \centering
  \includegraphics[scale=0.6]{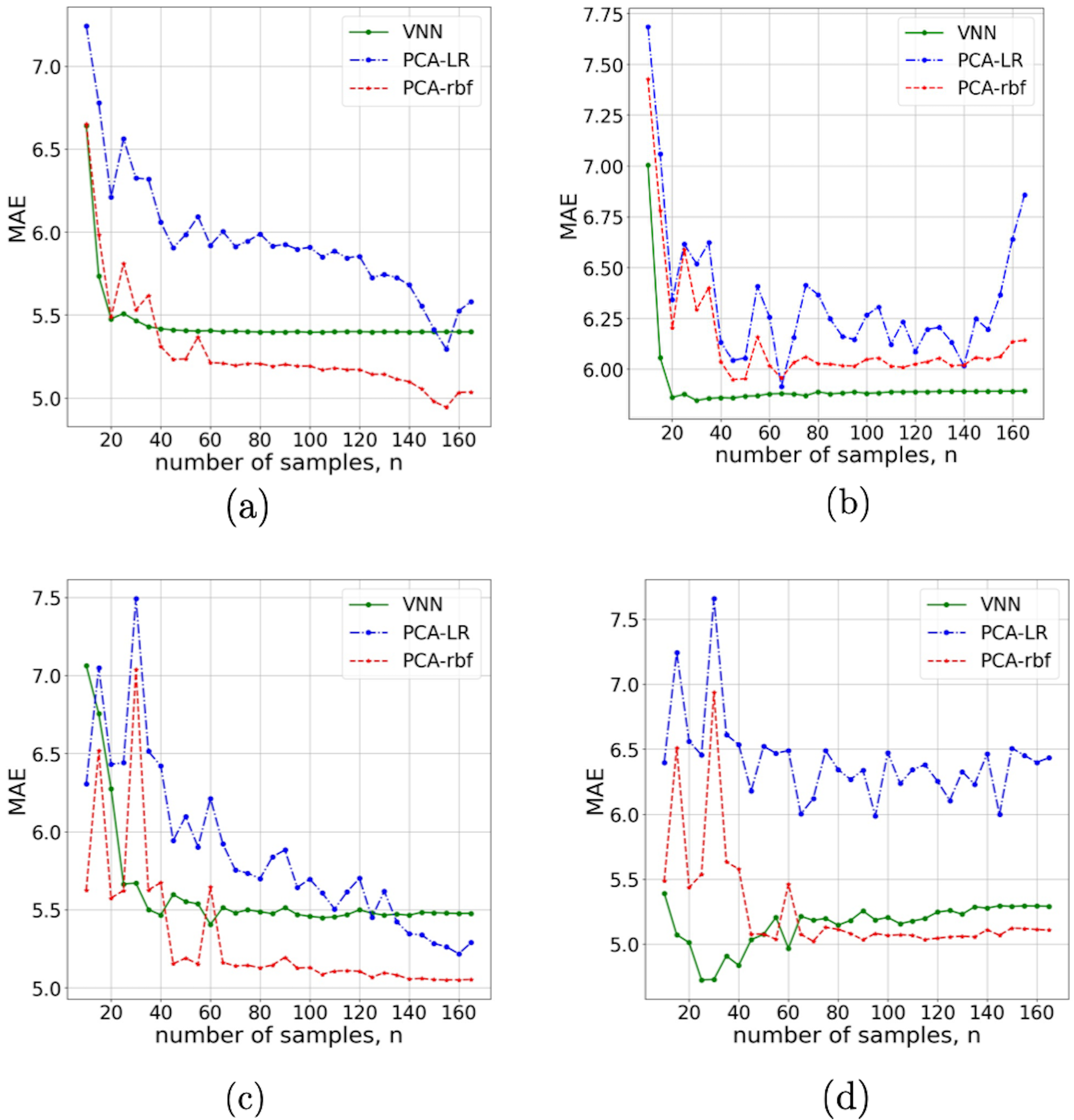}
   \caption{Stability on FTDC100 ((a) for training set and (b) for test set) and FTDC300 ((c) for training set and (d) for test set).}
   \label{FTDC_sta}
\end{figure}
\subsection{Stability of VNNs on Synthetic Data}\label{cmpr_synth}
We consider two settings for synthetic data. 

{\bf Friedman regression problem}:
This setting is described as `Friedman1' in~\cite{breiman1996bagging} and is generated using the routine \texttt{sklearn.datasets.make\_friedman1} in python. In the simulated dataset, we have $m$ independent features or predictors, each sampled uniformly from the range $[0,1]$. Out of $m$ predictors, any $5$ are used to generate the response variable and others are independent of the response. If the $5$ relevant predictors are given by $x_1,x_2,x_3,x_4$, and $x_5$, they are related to the response $y$ as
\begin{align}
    y = 10\sin(\pi x_1 x_2) +20(x_3 - 0.5)^2 +10x_4 + 5x_5 + \vartheta\;,
\end{align}
where $\vartheta$ represents noise. We generate a dataset of $n=1000$ samples using $m = 100$ features and noise distributed according to ${\cal N}(0,1)$. Figure~\ref{friedman1} a) shows the distribution of eigenvalues in the covariance matrix for the dataset used. The dataset is split into a $90/10$ train/test split and a sample covariance matrix $\hat\bC_{900}$ is evaluated from the features in the training set. Next, we perform regression against the response using VNN, {\sf PCA-LR}, and {\sf PCA-rbf} models. For VNNs, we use the same architecture and hyperparameters as described in Section~\ref{perturb}. Using different permutations of the training set, we train $100$ nominal models for VNN, {\sf PCA-LR}, and {\sf PCA-rbf}.

Figure~\ref{friedman1} b) plots the variation in mean performance of the nominal models with respect to perturbations in the sample covariance matrix. The last data point in Fig.~\ref{friedman1} b) corresponds to the performance of nominal models trained on $\hat\bC_{900}$. When $\hat\bC_{900}$ is replaced with $\hat\bC_{n'}$ for $n'\in [5,899]$, we observe that the VNN performance retains stability for sample covariance matrices generated with $n'>300$ samples. However, considerable randomness is introduced in the performance of PCA-regression models when the principal components are evaluated from $\hat\bC_{n'}$ for $n'\neq 900$. Similar phenomenon is observed for the performance on the test set in Fig.~\ref{friedman1} c). 

\begin{figure}[t]
  \centering
  \includegraphics[scale=0.35]{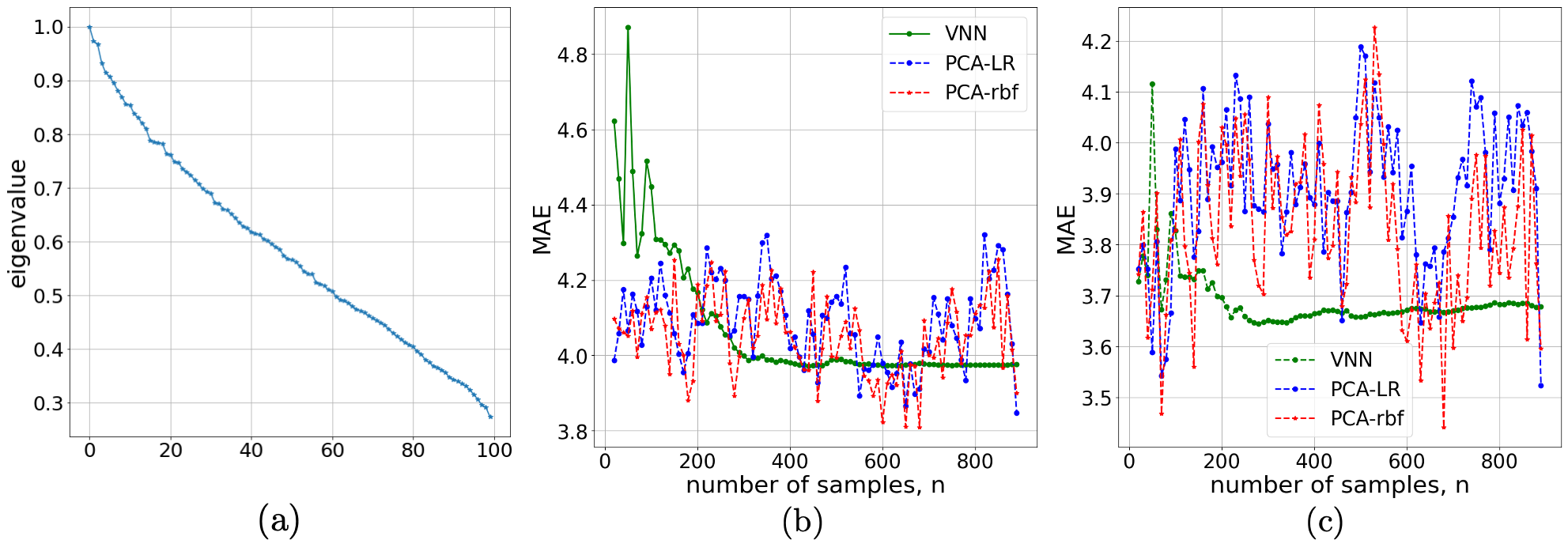}
   \caption{Friedman regression problem.}
   \label{friedman1}
\end{figure}

{\bf Linear regression problem}:
In this setting, we generate random linear regression problems using the routine \texttt{sklearn.datasets.make\_regression} in python, which allows us to specify the number of informative features; effective rank of the input dataset, i.e., the approximate number of singular vectors required to explain most of the input data by linear combinations; tail parameter which is the relative importance of the fat noisy tail of the singular values profile; and noise. In our experiments, we set the input dimension $m=100$, the dimension of the response to be $1$, number of samples $n=1000$, number of informative features to be $20$, effective rank of the input dataset to be $25$ and noise to be distributed as ${\cal N}(0,3)$. Since the stability of VNNs depends on the eigengap and kurtosis of underlying distribution of the data in direction of eigenvectors according to Theorem~\ref{vnn_stab}, we also aim to study the effects on VNN distribution with respect to variation in the strength of the tail of the eigenvalue distribution. To this end, we generate two datasets for linear regression with tail parameters set to $0.7$ and $0.2$. The eigenvalues of the covariance matrix for tail set to $0.7$ are more spread out as compared to that for tail set to $0.2$ (see Fig.~\ref{LR07} a) and Fig.~\ref{LR02} a)).

To evaluate the stability of VNNs, we split each dataset into a $90/10$ train/test split and generate sample covariance matrices $\hat\bC_{900}$ for both. Next, we perform regression using VNN and {\sf PCA-LR} models. For VNNs, we use the same architecture and hyperparameters as described in Section~\ref{perturb}. Using different permutations of the training set, we obtain $100$ nominal models for VNN and ${\sf PCA-lR}$. Figure~\ref{LR07} b) plots the variations in mean MAE performances of the nominal models with respect to perturbations in the sample covariance matrix, with the last data point corresponding to $n=900$, i.e., $\hat\bC_{900}$. When $\hat\bC_{900}$ is replaces with $\hat\bC_{n'}$ for any $n'\in[5,899]$, our experiments show that VNN performance is stable but significant randomness is induced into the performance of {\sf PCA-LR} model. Similar phenomenon is observed for the performance on the test set in Fig.~\ref{LR07} c). Same discussion also follows for the results in Fig.~\ref{LR02} b)-c), where we have set the tail parameter for eigenvalue distribution to be $0.2$. Comparison of MAE performances in Fig.~\ref{LR07} c) and Fig.~\ref{LR02} c) reveals that both VNN and ${\sf PCA-LR}$ perform significantly better for dataset with tail set to $0.2$ than that for $0.7$. This observation is along the expected lines as the first $25$ eigenvalues and eigenvectors that contain the most information are more separated (and hence, more accurately estimated by the sample covariance matrix) for the dataset with tail parameter $0.2$ as compared to the dataset with tail parameter $0.7$. Moreover, VNNs retain the stability property for both datasets. 

Finally, we also note that the observations above also extend to data with higher dimensions. Figure~\ref{LR071000} illustrates the results on a setting for linear regression with $m=1000$, number of informative features to be $20$, effective rank of the input dataset to be $25$ and noise to be distributed as ${\cal N}(0,3)$. The tail is set to $0.7$. The number of samples is $5000$ of which $4500$ are used to form the covariance matrix ${\hat \bC}_{4500}$ for training PCA-regression and VNN models. Figure~\ref{LR071000} a) illustrates the eigenvalue distribution for this setting. Figure~\ref{LR071000} b) and c) show the results of stability analysis by perturbing ${\hat \bC}_{4500}$ for training set and test set, respectively. 
\begin{figure}[t]
  \centering
  \includegraphics[scale=0.3]{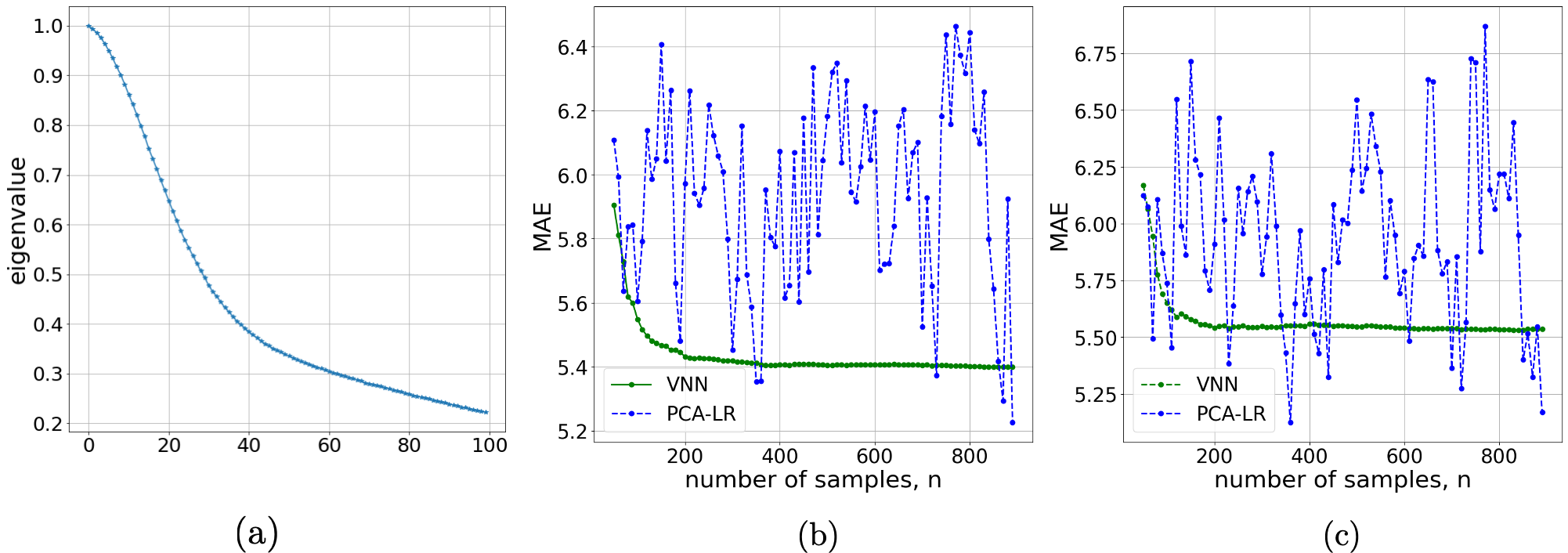}
   \caption{Stability of VNNs on linear regression problem (tail = 0.7).}
   \label{LR07}
\end{figure}

\begin{figure}[t]
  \centering
  \includegraphics[scale=0.3]{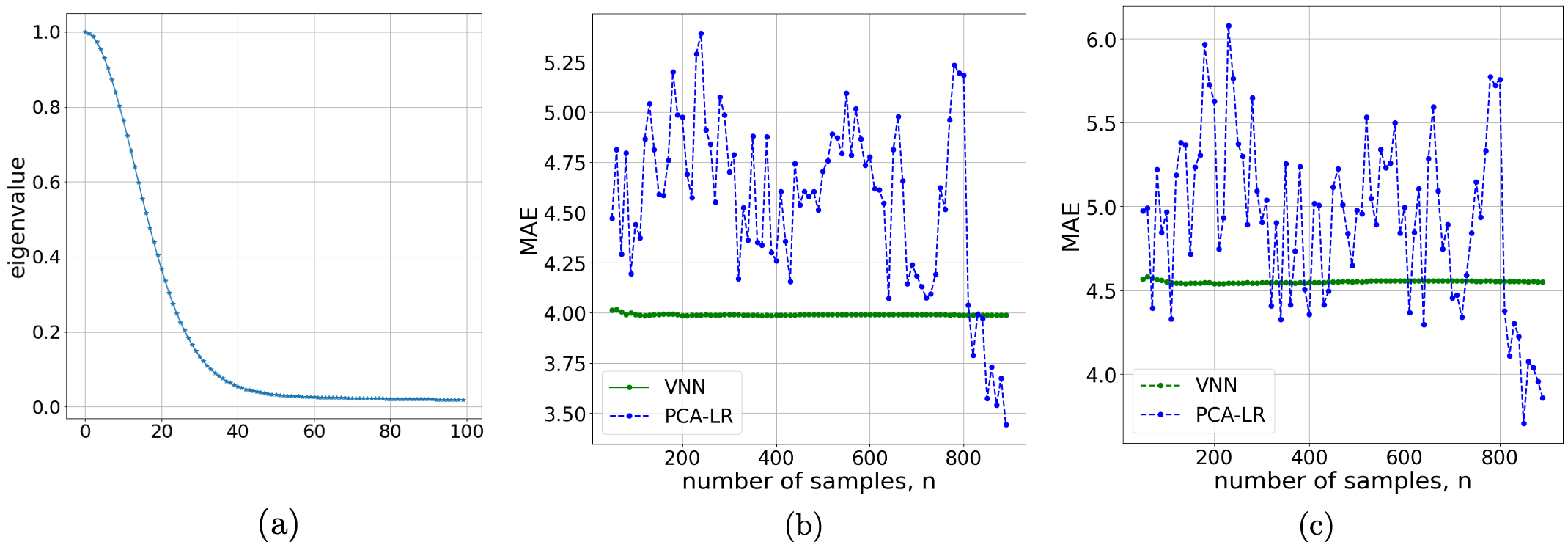}
   \caption{Stability of VNNs on linear regression problem (tail = 0.2).}
   \label{LR02}
\end{figure}

\begin{figure}[H]
  \centering
  \includegraphics[scale=0.3]{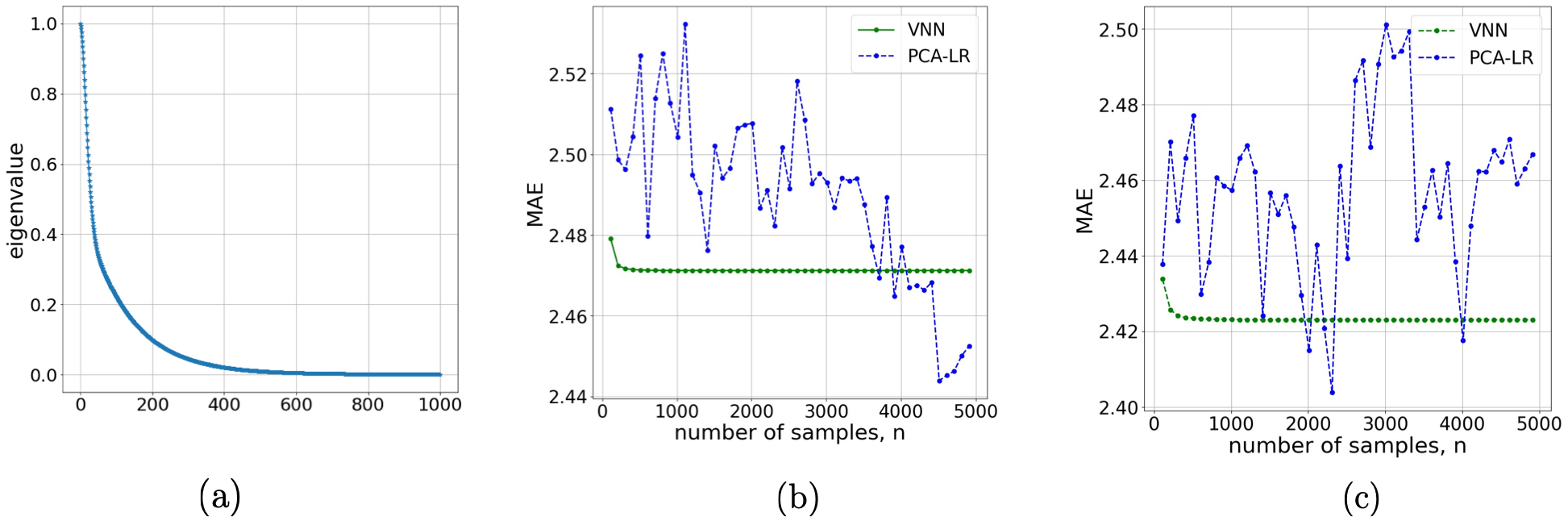}
   \caption{Stability of VNNs on linear regression problem with $m=1000$ (tail = 0.7).}
   \label{LR071000}
\end{figure}
\begin{figure}[H]
  \centering
  \includegraphics[scale=0.3]{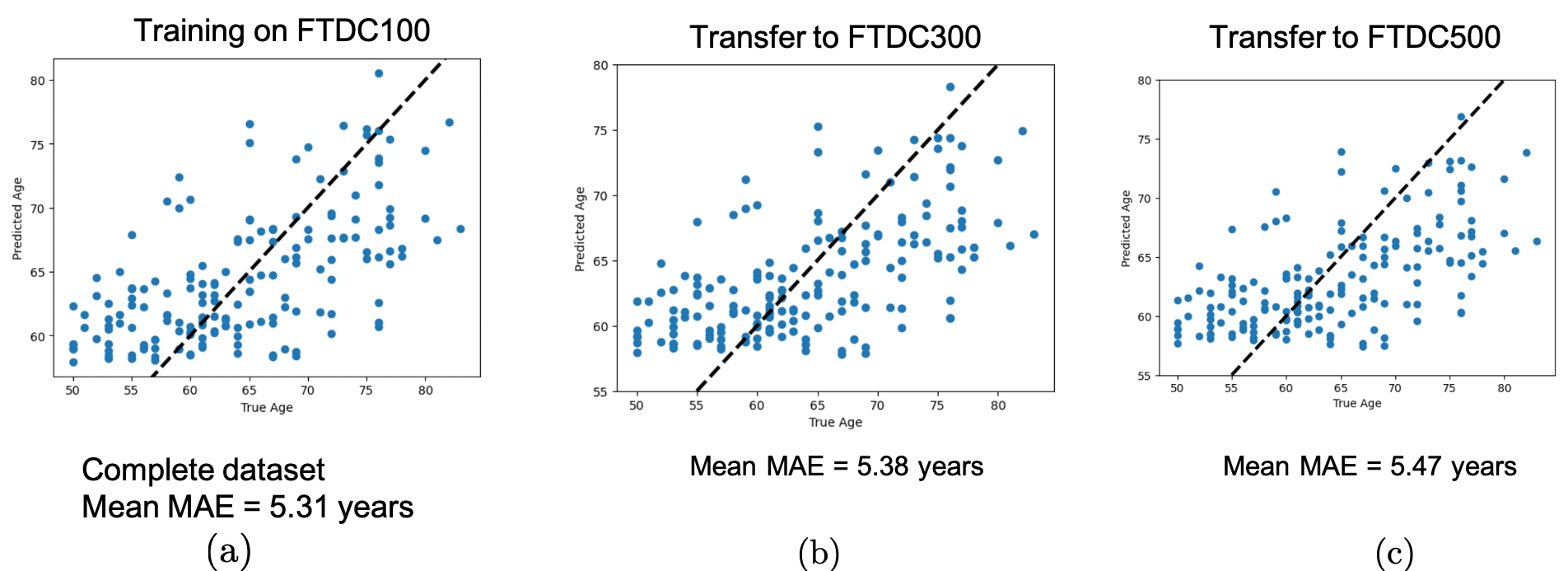}
   \caption{Transferability of VNN trained on FTDC100 to FTDC300 and FTDC500 datasets.}
   \label{transfer_supp}
\end{figure}
\subsection{Transferability of VNNs}\label{transfer_figs}
In Fig.~\ref{transfer_supp}, we show the predicted age vs true age plots corresponding to the first rows of the matrices in Fig.~\ref{transfer_fig}, i.e., when the VNN is trained on FTDC100 dataset and its transferability is evaluated on FTDC300 and FTDC500 datasets. Similar plots are observed for transferability in other reported settings as well.

\end{document}